\documentclass{article}

\usepackage[preprint]{neurips_2023}

\usepackage[utf8]{inputenc} 
\usepackage[T1]{fontenc}    
\usepackage{hyperref}       
\usepackage{url}            
\usepackage{booktabs}       
\usepackage{amsfonts,amsmath,amsthm}       
\usepackage{nicefrac}       
\usepackage{microtype}      
\usepackage{xcolor}         

\usepackage{enumitem}
\usepackage{bbm}

\usepackage{chngcntr}
\usepackage{cleveref}

\def\ddefloop#1{\ifx\ddefloop#1\else\ddef{#1}\expandafter\ddefloop\fi}
\def\ddef#1{\expandafter\def\csname bb#1\endcsname{\ensuremath{\mathbb{#1}}}}
\ddefloop ABCDEFGHIJKLMNOPQRSTUVWXYZ\ddefloop
\def\ddef#1{\expandafter\def\csname c#1\endcsname{\ensuremath{\mathcal{#1}}}}
\ddefloop ABCDEFGHIJKLMNOPQRSTUVWXYZ\ddefloop

\def\R{\mathbb{R}}

\def\Z{\mathbb{Z}}

\def\eps{\epsilon}

\def\D{\textup{d}}

\newcommand{\ip}[2]{\left\langle #1, #2 \right \rangle}

\newcommand{\sj}[1]{{\color{magenta}{\bf[SJ:} #1{\bf]}}}

\newtheorem{theorem}{Theorem}
\newtheorem{lemma}{Lemma}
\newtheorem{definition}{Definition}
\newtheorem{conjecture}{Conjecture}
\newtheorem{remark}{Remark}
\newtheorem{corollary}{Corollary}
\newtheorem{assumption}{Assumption}

\newcounter{parentnumber}
\makeatletter 
\newenvironment{subassumption}[1]{%
  \counterwithin*{assumption}{parentnumber}
  \def\subassumptioncounter{#1}%
  \refstepcounter{#1}%
  \protected@edef\theparentnumber{\csname the#1\endcsname}%
  \setcounter{parentnumber}{\value{#1}}%
  \setcounter{#1}{0}%
  \expandafter\def\csname the#1\endcsname{\theparentnumber.\Alph{#1}}%
  \ignorespaces
}{%
  \setcounter{\subassumptioncounter}{\value{parentnumber}}%
  \counterwithout*{assumption}{parentnumber} 
  \ignorespacesafterend
}
\makeatother

\crefname{assumption}{Assumption}{Assumptions}

\title{Limits, approximation and size transferability for GNNs on sparse graphs via graphops}

\author{%
  Thien Le\\
  CSAIL\\
  MIT\\
  Cambridge, MA\\
  \texttt{thienle@mit.edu} \\
  \And
  Stefanie Jegelka \\
  CSAIL \\
  MIT\\
  Cambridge, MA \\
    \texttt{stefje@csail.mit.edu} 
}

\definecolor{darkblue}{rgb}{0.0,0.0,0.65}
\definecolor{darkred}{rgb}{0.68,0.05,0.0}
\definecolor{darkgreen}{rgb}{0.0,0.29,0.29}
\definecolor{darkpurple}{rgb}{0.47,0.09,0.29}
\usepackage{hyperref}
\hypersetup{
   colorlinks = true,
   citecolor  = darkblue,
   linkcolor  = darkred,
   filecolor  = darkblue,
   urlcolor   = darkblue,
 }

\begin{document}

\maketitle

\begin{abstract}

  Can graph neural networks generalize to graphs that are different from the graphs they were trained on, e.g., in size? In this work, we study this question from a theoretical perspective. While recent work established such transferability and approximation results via graph limits, e.g., via graphons, these only apply nontrivially to dense graphs. To include frequently encountered sparse graphs such as bounded-degree or power law graphs, we take a perspective of taking limits of operators derived from graphs, such as the aggregation operation that makes up GNNs. This leads to the recently introduced limit notion of graphops \citep{backhausz_szegedy_2022}. We demonstrate how the operator perspective allows us to develop quantitative bounds on the distance between a finite GNN and its limit on an infinite graph, as well as the distance between the GNN on graphs of different sizes that share structural properties, under a regularity assumption verified for various graph sequences. Our results hold for dense and sparse graphs, and various notions of graph limits.

\end{abstract}

%


\section{Introduction}\label{introduction}

Since the advent of 
graph neural networks (GNNs), 
deep learning has become one of the most promising tools to address graph-based learning tasks \citep{gilmer_mpnn, scarselli_gnn, kipf_welling_semi_sup, geometric_dl}. 
Following the mounting success of applied GNN research, theoretical analyses have been following. For instance, many works study GNNs' representational power \citep{azizian21,morris19,xuhlj19,xu_algorthmic_alignment, garg20,chen20count, maron_provably_powerful_gnn,loukas20,loukas20dist,abboud21}.

A hitherto less addressed question of practical importance is the possibility of size generalization, i.e., transferring a learned GNN to graphs of different sizes  \citep{
ruiz_transferability, cnn_transferability,
xu_extrapolate, yehudai21,chuang22tmd,roddenberry_local_distribution_graph_signal_processing}, especially for sparse graphs. 
For instance, it would be computationally desirable to train a GNN on small graphs and apply it to large graphs. This question is also important to judge the reliability of the learned model on different test graphs.
To answer the size generalization question, we need to understand under which conditions such transferability is possible -- since it may not always be possible \citep{xu_extrapolate, yehudai21,jegelka22icm} -- and what output perturbations we may expect. 
For a formal analysis of perturbations and conditions, we need a suitable graph representation that captures inductive biases and allows us to compare models for graphs of different sizes. 
\emph{Graph limits} can help to formalize this, as they help understand biases as the graph size tends to infinity.

Formally, \emph{approximation theory} asks for bounds between a GNN on a finite graph and its infinite counterpart, while \emph{transferability} compares model outputs on graphs of different sizes.
The quality of the bounds depends on how the two GNNs (and corresponding graphs) are intrinsically linked, in particular, to what extent the graphs share relevant structure. This yields conditions for size generalization. For example, the graphs could be sampled from the same graph limit \citep{ruiz_transferability} or from the same random graph model \citep{keriven_gnn_from_model}.

In particular, \citet{ruiz_transferability} study approximation and transferability via the lens of \emph{graphons} \citep{lovasz_graph_limit, lovasz_szegedy_2006}, which characterize the limits of \emph{dense} graphs. Yet, many real-world graphs are not dense, for instance, planar traffic networks, power law graphs, Hamming graphs (including hypercubes for error-correcting code), or grid-like graphs e.g., for images. 
For \emph{sparser} graphs, the correct notion of limit suitable for deep learning is still an open problem, as typical bounded-degree graph limits such as the Benjamini-Schramm limit of random rooted graphs \citep{benjamini_schramm}, or graphings \citep{lovasz_graph_limit} are less well understood and often exhibit pathological behaviors (see Section \ref{section:warmup}). Limits of intermediate graphs, such as the hypercubes, are even more obscure. 
Hence, understanding limits, inductive biases and transferability of GNNs for sparse graphs remains an open problem in understanding graph representation learning.
%

This question is the focus of this work. To obtain suitable graph limits for sparse graphs and to be able to compare GNNs on graphs of different sizes while circumventing challenges of sparse graph limits, we view a graph as an \emph{operator} derived from it. This viewpoint is naturally compatible with GNNs, as they are built from convolution/aggregation operations. We show how the operator perspective allows us to define limits of GNNs for infinite graphs. We achieve this by exploiting the recently defined notion of \emph{graphop}, which generalizes graph shift operators, and the \emph{action convergence} defined in the space of graphops \citep{backhausz_szegedy_2022}.
Our definition of GNN limits enables us to prove rigorous bounds for approximation and transferability of GNNs for sparse graphs. Since graphops encompass both graphons and graphings, we generalize similar bounds for graphon neural networks \citep{ruiz_transferability} to a much wider set of graphs.

Yet, using graphops requires technical work. For instance, we need to introduce an appropriate discretization of a graphop to obtain its corresponding finite graph shift operators. We use these operators to define a generalized graphop neural network that acts as a limit object, with discretizations that become finite GNNs. Then we prove approximation and transferability results for both the operators (graphops and their discretizations) and GNNs.

\textbf{Contributions.}
%
To the best of our knowledge, this is the first paper to provide approximation and transferability theorems specifically for sparse graph limits. Our main tool, graphops, has not been used to study GNNs before, although viewing graphs as operators is a classic theme in the literature. Our specific contributions are as follows: 
\begin{enumerate}[topsep=0pt,itemsep=0ex,partopsep=1ex,parsep=1ex,leftmargin=0.5cm]
  \item We define a \emph{graphop convolution}, i.e., an operator that includes both finite graph convolutions and a limit version that allows us to define   
a limit object for GNNs applied to graphs of size $n \to \infty$.
%
  \item We rigorously prove an approximation theorem (Theorem \ref{theorem:approximation_via_discretization}) that bounds a distance between a graphop $A$ (acting on infinite-dimensional $\Omega$) and its discretization $A_n$ (acting on $\mathbb{R}^n$), in the $d_M$ metric introduced by \citet{backhausz_szegedy_2022}. To do so, we introduce an appropriate discretization. Our result applies to a more general set of nonlinear operators, and implies a transferability bound between finite graphs (discretizations) of different sizes.
  \item For neural networks, we present a quantitative approximation and transferability bound that guarantees outputs of graphop neural networks are close to those of the corresponding GNNs (obtained from discretization). 
\end{enumerate}
\subsection{Related work}

The closest related work is \citep{ruiz_transferability}, which derives approximation and transferability theorems for \emph{graphon} neural networks, i.e.,   \emph{dense} graphs. 
For graphons, the convolution kernel has a nice spectral decomposition, 
which is exploited by \citet{ruiz_transferability}. In contrast, \emph{sparse} graph limits are not known to enjoy nice convergence of the spectrum \citep{backhausz_szegedy_2022,aldous_lyon_conjecture}, so we need to use  different techniques.
%
Since the notion of graphop generalizes both dense graph limits and certain sparse graph limits, our results apply to dense graphs as well. 
Our assumptions and settings are slightly different from \cite{ruiz_transferability}. For instance, they allow the convolution degree $K \to \infty$ and perform the analysis in the spectral domain, whereas our  $K$ is assumed to be a fixed finite constant. 
As a result, their bound has better dependence of $O(1/n)$ on $n$--the resolution of discretization, but does not go to  $0$ as  $n \to \infty$. Ours have extra dependence on  $K$ and a slower rate of $O(n^{-1 / 2}$ but our bounds go to  $0$ as  $n \to \infty$. 


Other works use other notions than graph limits to obtain structural coherence. 
\citet{cnn_transferability} obtain a transferability result for spectral  graph convolution  networks via analysis in frequency domains. They sample finite graphs from general topologies as opposed to a graph limit. Their graph signals are \emph{assumed} to have finite bandwidth while ours is only assumed to be in $L^2$. Their signal discretization scheme is assumed to be close to the continuous signals, while ours is proven to be so.  
\citet{roddenberry_local_distribution_graph_signal_processing} address sparse graphs and give a transferability bound between the loss functions of two random rooted graphs. 
However, the metric under which they derive their result is rather simple: if the two graphs are not isomorphic then their distance is constant, otherwise, they use the Euclidean metric between the two graph signals. This metric hence does not capture combinatorial, structural differences of functions on non-isomorphic graphs.  
To study transferability, \citet{keriven_gnn_from_model} 
sample from standard random graph models, as opposed to a graph limit, 
resulting in a bound of order $O(n^{-1 / 2})$, which is similar to ours. 
They also need an assumption on the closeness of the graph signal. 


\section{Background}

\paragraph{Notation} Let $\mathbb{N}$ be $\{1,2,\ldots\}$ and write $[n] = \{1,\ldots,n\}$ for any $n \in \mathbb{N}$. For a scalar $\alpha \in \R$ and a set $S \subset \R$, let $\alpha S = \{\alpha s : s \in S\}$. The abbreviation a.e. stands for `almost everywhere'.

For a measure space $(\Omega, \mathcal{B}, \mu)$ and $p \in [1, \infty]$, denote by $L^p(\Omega)$ the corresponding $L^p$ function spaces with norm $\|\cdot\|_p: f \mapsto (\int_{\Omega} |f|^p d\mu )^{1 / p}$. For any $p, q \in [1,\infty]$, define the operator norms $\|\cdot\|_{p \to q}: A \mapsto \sup_{v \in L^\infty} \|vA\|_{q}/\|v\|_{p}$. 

For function spaces, we use $\mathcal{F} = L^2([0,1])$ and $\mathcal{F}_n = L^2([n] / n)$, for any $n \in \mathbb{N}$. For any $L^p$ space $\mathcal{H}$, denote by $\mathcal{H}_{[-1,1]}$ the restriction to functions with range in $[-1,1]$ a.e. and  $\mathcal{H}_{\text{Lip}(L)}$ the restriction to functions that are $L$-Lipschitz a.e. and $\mathcal{H}_{\text{reg}(L)} = \mathcal{H}_{[-1,1]}\cap \mathcal{H}_{\text{Lip}(L)}$. 

\paragraph{Graph neural networks (GNNs)}
GNNs are functions that use graph convolutions to incorporate graph structure into neural network architectures. Given a finite graph $G = (V,E)$ and a function  $X: V \to \mathbb{R}$ (called \emph{graph signal} or \emph{node features}), a GNN $\Phi_F$ ($F$ for `finite') with $L$ layers, $n_i$ neurons at the $i$-th layer, nonlinearity $\rho$ and learnable parameters $h$, is:
\begin{align}
  \Phi_F(h, G, X) &= X_L(h, G, X),\\
  \label{eqn:gnn}
  \left[X_l(h, G, X)\right]_f &= \rho\bigg( \sum_{g = 1}^{n_{l - 1}} A_{l,f,g}(h,G) [X_{l - 1}]_g \bigg) ,\qquad l \in [L], f\in [n_l]\\
                X_0(h,G,X) &= X,
\end{align}
where $[X_l]_{f}$ is the output of the $f$-th neuron in the $l$-th layer, which is another graph signal. The input graph information is captured through order $K$ \emph{graph convolutions} $A_{l,f,g}(h,G) := \sum_{k = 0}^K h_{l,f,g,k} GSO(G)^k$, where $GSO(G)$ is a \emph{graph shift operator} corresponding to $G$ --- popular examples include the adjacency matrix or the Laplacian \citep{kipf_welling_semi_sup,cnn_transferability}. The power notation is the usual matrix power, while the notation $h_{l,f,g,k}$ highlights that there is a learnable parameter for each convolution order $k$, between each neuron $f$ and $g$ from layer $l-1$ to layer $l$ of the neural network. Thus, the number of learnable parameters in a GNN does not depend on the number of vertices of the graph used to form the GSO. 

\subsection{Graph limits}\label{section:warmup}

Graph limit theory involves embedding discrete graphs into spaces with rich underlying topological and/or geometric structures and studying the behavior of convergent (e.g. in size) graph sequences. 

\paragraph{Dense graphs} A popular example of graph limits are graphons - symmetric $L^1([0,1]^2)$ (Lebesgue-measurable) functions 
whose value at $(x,y)$ can be thought of (intuitively) as 
the weight of the $xy$-edge in a graph with vertices in $[0,1]$. 
When equipped with the \emph{cut metric} (see Appendix \ref{appendix:additional_notations} for the exact definition), it contains limits of convergent sequences of dense graphs.  Convergence in this space is dubbed \emph{dense graph convergence} because for any $W \in L^1([0,1]^2),\|W\|_{\square} = 0$ iff $W = 0$ outside a set of Lebesgue measure $0$. 
This implies that graphs with a subquadratic number of edges, such as grids or hypercubes, are identified with the empty graph in the cut norm. Dense graph convergence is very well understood theoretically and is the basis for recent work on GNN limits \citep{ruiz_transferability}. 


\paragraph{Sparse graphs} There are two notions of limits that are deemed suitable for bounded-degree graph convergence in the literature. The first, \emph{graphing} \citep{lovasz_graph_limit}, is a direct counterpart of a graphon. Recall that graphons are not suitable for sparse graphs because the Lebesgue measure on $L^2([0,1]^2)$ is not fine enough to detect edges of bounded-degree graphs. Therefore, one solution is to consider other measure spaces. Graphings are quadruples  $(V, \mathcal{A}, \lambda, E)$ where $V$ and  $E$ are interpreted as the usual vertex and edge sets and  $(V,\mathcal{A},\lambda)$ together form a Borel measure such that $E$ is in $\mathcal{A} \times \mathcal{A}$ satisfying a symmetry condition. While Lebesgue measures are constructed from a specific topology of open sets on $\R$, for graphings, we are allowed the freedom to choose a different topological structure (for instance a \emph{local topology}) on $V$. 
The definition of graphings is theoretically elegant but harder to work with since the topological structures are stored in the $\sigma$-algebra. A second way to embed bounded-degree graphs, Benjamini-Schramm limits \citep{benjamini_schramm}, uses distributions over random rooted graphs. Roughly speaking, for each $k \in \mathbb{N}$, one selects a root uniformly at random from the vertex set in the graph and considers the induced subgraph of vertices that are at distance at most $k$ from the root. The randomness of the root induces a distribution over the set of rooted graphs of radius $k$ from the root. 

Graphings and distributions of random rooted graphs are intimately connected, but their connection to convergent bounded-degree graph sequences is not well-understood. For example, a famous open conjecture by \cite{aldous_lyon_conjecture} asks whether all graphings are weak local limits of some sequence of bounded-degree graphs.
  The unresolved conjecture of Aldous and Lyons means that one cannot simply take an arbitrary graphing and be guaranteed a finite bounded-degree graph sequence converging to said graphing, which is the main approach in \cite{ruiz_transferability} for dense graphs. 
A self-contained summary of graphings within the scope of this paper is provided in Appendix \ref{appendix:graphings}. Infinite paths and cycles also have nice descriptions in terms of graphings (also in Appendix \ref{appendix:graphings}), which we will use in our constructions for Lemma \ref{lemma:constant_to_constant_graphs}. 

\subsection{Graphops and comparing across graph sizes}\label{section:prelim}

More recently, \citet{backhausz_szegedy_2022} approach graph limits from the viewpoint of limits of operators, called \emph{graphops}. This viewpoint is straightforward for finite graphs: both the adjacency matrix and Laplacian, each defining a unique graph, are linear operators on $\R^{\# \text{vertices}}$. Moreover, viewing graphs as operators is exactly what we do with GSOs and graph convolutions. Hence, graphop seems to be an appropriate tool to study  GNN approximation and transferability. 
On the other hand, there are challenges with this approach: being related to graphings, they inherit some of graphings' limitations, such as the conjecture of \cite{aldous_lyon_conjecture}. Moreover, 
to understand GNN transferability from size $m$ to  $n$, one needs to compare an $m\times m$ matrix with an  $n \times n$ matrix, which is nontrivial. This is done by comparing their actions on $\R^{m}$ versus $\R^{n}$. 
It turns out that these actions, under an appropriate metric, define a special mode of operator convergence called \emph{action convergence}. The resulting limit objects are well-defined and nontrivial for sparse graphs and intermediate graphs, while also generalizing dense graphs limits. We will describe this mode of convergence, the corresponding metric, and our own relaxation of it later in this section. 

We now describe how graphs of different sizes can be compared through the actions of their corresponding operators on some function spaces. 

\paragraph{Nonlinear $P$-operators}   For an $n$-vertex graph, its adjacency matrix, Laplacian, or Markov kernel of random walks are examples of linear operators on $L^p([n] / n)$. To formally generalize to the infinite-vertex case, \citet{backhausz_szegedy_2022} use \emph{$P$-operators}, which are linear operators from $L^\infty(\Omega)$ to $L^1(\Omega)$ with finite $\|A\|_{\infty \to 1}$. In this paper, we further assume they have  finite  $\|\cdot\|_{2 \to 2}$ norm but are not necessarily linear. We address these deviations in Section \ref{section:deviation}.

\paragraph{Graphops} $P$-operators lead to a notion of graph limit that applies to both dense and sparse graphs. \emph{Graphops} \citep{backhausz_szegedy_2022} are positivity-preserving (action on positive functions results in positive functions), self-adjoint $P$-operators. Adjacency matrices of finite graphs, graphons \citep{lovasz_szegedy_2006}, and graphings \citep{lovasz_graph_limit} are all examples of graphops. Changing the positivity-preserving requirement to positiveness allows one to consider Laplacian operators. 

\paragraph{$(k,L)$-profile of a $P$-operator} Actions of graphops are formally captured through their $(k,L)$-profiles, and these will be useful to compare different graphops. Pick $k \in \mathbb{N}$, $L \in [0,\infty]$ and $A$ a $P$-operator on $(\Omega, \mathcal{B}, \mu)$. Intuitively, we will take $k$ samples from the space our operators act on, apply our operator to get  $k$ images, and concatenate  samples and images into a joint distribution on  $\R^{2k}$, which gives us one element of the profile. For instance, for $n$-vertex graphs, the concatenation results in a matrix $M \in \mathbb{R}^{n \times 2k}$, so each joint distribution is a sum (over rows of $M$) of $n$ Dirac distributions. In the limit, the number of atoms in each element of the profile increases, and the measure converges (weakly) to one with density. More formally, denote by $\cD(v_1,\ldots,v_k)$ the pushforward of $\mu$ via  $x \mapsto (v_1(x),\ldots,v_k(x))$ for any tuple $(v_i)_{i \in [k]} \in L^2(\Omega)$. The \emph{$(k,L)$-profile} of  $A$ is:
\begin{align}
  \cS_{k,L}(A) := \{\cD(v_1,\ldots,v_k, Av_1, \ldots, Av_k) :  v_i \in L^{\infty}_{\text{reg}(L)}(\Omega), i = 1\ldots k\}.
\end{align}
   Formally, denote by $\cP(\mathbb{R}^k)$ the set of  Borel probability distributions over $\R^{2k}$. Regardless of the initial graph size, or the space on which the operators act, $(k,L)$-profiles of  $A$ are always some subsets of $\cP(\R^{2k})$ which allow us to compare operators acting on different spaces. 

   

\paragraph{Convergence of $P$-operators} We compare two profiles (closed subsets $X, Y \subset \cP(\R^{2k})$) via a Hausdorff metric
$  d_H(X, Y) := \max(\sup_{x \in X} \inf_{y \in Y} d_{LP}(x,y),
            \sup_{y \in Y} \inf_{x \in X} d_{LP}(x,y)).
$ 
 Here, $d_{LP}$ is the L\'evy-Prokhorov metric on $\cP(\R^{2k})$ (see exact definition in Appendix \ref{appendix:additional_notations}), which 
%
%
%
metrizes weak convergence of Borel probability measures, and translates action convergence to weak convergence of measures. Finally, given  any two $P$-operators $A,B$, we can compare their profiles across all different $k$ at the same time as
\begin{equation}
  d_M(A,B) := \sum_{k = 1}^\infty 2^{-k} d_H(\mathcal{S}_{k,L}(A), \mathcal{S}_{k,L} (B)).
\end{equation}
Intuitively, we allow $d_H$ to grow subexponentially in  $k$ by the scaling $2^{-k}$. Our definition of profile slightly differs from that of \citet{backhausz_szegedy_2022}, using $L^\infty_{\text{reg}(L)}$ instead of their $L^\infty_{[-1,1]}$. 
   However, we will justify this deviation in Section \ref{section:deviation}, Theorem \ref{theorem:growing_profiles}: by letting $L$ grow slowly in  $n$, we recover the original limits in \citet{backhausz_szegedy_2022}.


This \emph{action convergence} turns out to be one of the `right' notions of convergence that capture both sparse and dense graph limits, as well as some intermediate density graphs:
\begin{theorem}[Theorem 1.1 \cite{backhausz_szegedy_2022}]
  Convergence under $d_M$  is equivalent (results in the same limit) to dense graph convergence when restricted to graphons and equivalent to local-global convergence when restricted to graphings. 
\end{theorem}


\section{Graphop neural networks}\label{section:graphopnn}

Graph limits allow us to lift finite graphs onto the richer space of graphops to discuss convergent graph sequences $G_i \to G$. For finite GNNs (Eqn \eqref{eqn:gnn}), fixing the graph input $G_i$ and learnable parameter $h$ results in a function $\Phi_F(h,G_i,\cdot)$ that transforms the input graph signal (node features) into an output graph signal. The transferability question asks how similar $\Phi_F(h,G_i,\cdot)$ is to $\Phi_F(h,G_j,\cdot)$ for some $i \neq j$. In our approach using approximation theory, we will compare both functions to the limiting function on $G$. This is done by an appropriate lift of the GNN onto a larger space that we call \emph{graphop neural networks}. 

We then introduce a discretization scheme of graphop neural networks to obtain finite GNNs, similar to graphon sampling \citep{ruiz_transferability} and sampling from topological spaces \citep{cnn_transferability}. Finally, Lemma \ref{lemma:discretization_of_selfadjoint} asserts that, restricted to self-adjoint  $P$-operators, discretizations of graphops are indeed graph shift operators (GSOs).

\subsection{Convolution and graphop neural networks}

Similar to how GSOs in a GNN act on graph signals, graphops act on some $L^{2}$ signals (called \emph{graphop signals}). The generalization is straightforward: replacing GSOs in the construction of the GNN in Eqn.\ \eqref{eqn:gnn} with graphops results in \emph{graphop convolution} and replacing graph convolution with graphop convolution gives \emph{graphop neural networks}. 

Formally, fix a maximum order $K \in \mathbb{N}$. For some measure space $(\Omega, \mathcal{B}, \mu)$, select a graphop $A: L^2(\Omega) \to L^2(\Omega)$ and a graphop signal $X \in L^2(\Omega)$. We define a \emph{graphop convolution} operator as a weighted sum of at most $K - 1$ applications of $A$:
\begin{equation}
  H(h,A)[X] := \sum_{k = 0}^{K-1}(h_k  A^{k})[X],
\end{equation}
where $h \in\R^K$ 
are (learnable) filter parameters and $A^k$
is the composition of $k$ duplicates of $A$. 
The square bracket $[v]_i$ indicates the $i$-th entry of a tuple $v$. 


For some number of layers $L \in \mathbb{N}, \{n_i\}_{i \in [L]} \in \mathbb{N}, n_0 := 1$, define a \emph{graphop neural network} $\Phi$ with $L$ layers and $n_i$ features in layer $i$ as: 
\begin{align}
  \Phi(h, A, X) &= X_L(h, A, X),\\\label{eqn:graphopnn}
  X_l(h, A, X) &= \left[\rho\bigg( \sum_{g = 1}^{n_{l - 1}} H(h^l_{f,g}, A) [X_{l - 1}]_g \bigg)\right]_{f \in [n_l]} ,\qquad l \in [L],\\
                X_0(h,A,X) &= X
\end{align}
with filter parameter tuple $h = (h^1,\ldots,h^L)$, $h^l \in (\R^{K})^{n_{l} \times n_{l - 1}}$ for any $l \in [L]$, and graphop signal tuple $X_l \in (L^2(\Omega))^{n_l}$ for any $l \in [L] \cup \{0\}$.  Eqn \eqref{eqn:graphopnn} and Eqn \eqref{eqn:gnn} are almost identical, with the only difference being the input/output space: graphops replacing finite graphs, and graphop signals replacing graph signals. 


\subsection{From graphop neural networks to finite graph neural networks}
 We are specifically interested in finite GNNs that are discretizations of a graphop (for instance finite grids as discretizations of infinite grids), so as to obtain a quantitative bound that depends on the resolution of discretization. 
 To sample a GNN from a given graphop  $A: \mathcal{F} \to \mathcal{F}$, we first sample a GSO and plug it into Eqn \eqref{eqn:gnn}. Choose a resolution $m \in \mathbb{N}$ and define the GSO $A_m$, for any graph signal $X \in \mathcal{F}_m$ as: 
\begin{align}
  A_m X(v) &:= m\int_{v - \frac{1}{m}}^v(A \widetilde{X})\D \lambda, \qquad v \in  [m] / m, \label{eqn:discretization}\\ 
  \Phi_m(h, A, X) &:= \Phi(h, A_m, X),
\end{align}
where graphop signal $\widetilde{X} \in \mathcal{F}$ is an extension of graph signal $X \in \mathcal{F}_m$ defined as  
\begin{align}
  \widetilde{X}(u) := X\left(\frac{\left\lceil um  \right\rceil  }{m}\right), \qquad u \in [0,1].
\end{align}
Intuitively, to find the image of the sampled GSO  $A_m$ when applied to a graph signal  $X$, we transform the graph signal into a graphop signal  $\tilde{X}$ by using its piecewise constant extension. We then apply the given $A$ to $\tilde{X}$ to get an output graphop signal on $L^{2}([0,1])$. Discretizing this graphop signal by dividing the domain $[0,1]$ into  $m$ partitions of equal size and integrating over each partition, one gets a graph signal on  $\mathcal{F}_m$ - the image of $A_m X$.
Note that if $A$ is linear then  $A_m$ is necessarily linear, but our definition of graphop does not require linearity. Therefore, $A_m$ is strictly more general than the matrix representation of graph shift operators. We have the following well-definedness result:
\begin{lemma}\label{lemma:discretization_of_selfadjoint}
If a graphop $A: \mathcal{F} \to \mathcal{F}$ is self-adjoint, then for each resolution $m \in \mathbb{N}$, the discretization $A_m: \mathcal{F}_m \to \mathcal{F}_m$ defined above is also self-adjoint.
\end{lemma}
The proof can be found in Appendix \ref{appendix:proof_from_section_graphopnn}. Compared to previous works, our discretization scheme in Eqn \eqref{eqn:discretization} looks slightly different. 
In \cite{ruiz_transferability}, given a graphon $W: [0,1]^2 \to \R$, the discretization at resolution $n$ was defined by forming the matrix $S \in \R^{n \times n}: S_{i,j} = W(i / n,j / n)$. A related discretization scheme involving picking the interval endpoints at random was also used, but the resulting matrix still takes values at discrete points in $W$. These two sampling schemes rely crucially on their everywhere continuous assumptions for the graphon $W$. Indeed, but for continuity requirements, two functions that differ only at finite discrete points $(i / n, j / n), i,j \in [n]$ are in the same $L^2$ class of functions, but will give rise to completely different samples. Furthermore, not every $L^2$ class of functions has a continuous representative. This means that our discretization scheme is strictly more general than that used by \cite{ruiz_transferability} even when restricted to graphons.  This difference comes from the fact that we are discretizing an operator and not the graph itself. For our purpose, taking values at discrete points for some limiting object of sparse graphs will likely not work, since sparsity ensures that most discrete points are trivial. The integral form of the discretization (as opposed to setting $A_mX(v) = (A \widetilde{X})(v)$ for all  $v \in [m] / m$, for example) is crucial for Lemma \ref{lemma:discretization_of_selfadjoint}.

\section{Main result: Approximation and transferability}\label{section:discretization}

\subsection{Results for \texorpdfstring{$P$}{P}-operators}

Our first set of theorems address approximation and transferability of $P$-operators: under certain regularity assumptions to be discussed later, $P$-operators are well approximated by their discretizations:
\begin{theorem}[Approximation theorem]\label{theorem:approximation_via_discretization}
  Let $A: \mathcal{F} \to \mathcal{F}$ be a $P$-operator satisfying Assumption \ref{assumption:lipschitz_map} with constant $C_A$; Assumption \ref{assumption:constant_to_constant} or \ref{assumption:lipschitz_to_lipschitz} with resolutions in $\mathcal{N}$. Fix $n \in \mathcal{N}$ and consider $(k, C_v)$-profiles. Let $A_n: \mathcal{F}_n \to \mathcal{F}_n$ be a discretization of  $A$ as defined in Equation \eqref{eqn:discretization}. Then:
  \begin{equation}
    d_M\left(A, A_n\right) \leq 2\sqrt{\frac{C_A C_v}{n}} + \frac{C_v + 1}{n}.
  \end{equation}\end{theorem}
  Compared to theorems in \cite{ruiz_transferability}, our explicit dependence on $n$ has an extra  $n^{- 1 / 2}$ term that stems from techniques used to bound the Lévy-Prokhorov distance between two entry distributions obtained from functions that differ by at most  $O(n^{-1})$ in  $L^2$ norm.

As an immediate corollary, invoking the triangle inequality yields a transferability bound.
\begin{corollary}[Transferability]\label{corollary:transferability}
  Let $A:\mathcal{F} \to \mathcal{F}$ be a  $P$-operator satisfying assumptions of Theorem \ref{theorem:approximation_via_discretization} with constant $C_A$ and resolutions $\mathcal{N}$. For any $n, m \in \mathcal{N}$, let $A_n: \mathcal{F}_n \to \mathcal{F}_n$  and  $A_m: \mathcal{F}_m \to \mathcal{F}_m$ be discretizations as defined in Equation \eqref{eqn:discretization}. Then:
\begin{align}
  d_M\left(A_m, A_n\right) \leq \left(m^{-\frac{1}{2}} + n^{-\frac{1}{2}}\right)2\sqrt{C_A C_v} + (m^{-1} + n^{-1}) (C_v + 1).
  \end{align}
\end{corollary}
  We emphasize that these theorems work for general nonlinear $P$-operators and not only the linear graphops defined in \citep{backhausz_szegedy_2022}. 
  
\paragraph{Proof sketch}
The  full proof of Theorem \ref{theorem:approximation_via_discretization} is in Appendix \ref{appendix:proofs_from_section_discretization}. To bound the distance in $d_M$  between two operators, for each sample size $k \in \mathbb{N}$, we give a bound on the Hausdorff metric $d_H$ between the two $(k,C_v)$-profiles. As long as the dependence on  $k$ of these bounds is polynomial, the infinite sum in the definition of  $d_M$ converges. We do this by picking an arbitrary distribution $\overline{\eta}$ from  $\mathcal{S}_{k,C_v}(A)$, which by definition is given by a $k$-tuple $F$ of functions in  $L^\infty_{\text{reg}(C_v)}$. Discretize each element of $F$ and consider its entry distribution results in  $\overline{\eta}_n \in \mathcal{S}_{k,C_v}(A_n)$. We show that we can give an upper bound of $d_{LP}(\overline{\eta}, \overline{\eta}_n)$ that is independent of the choice of  $\overline{\eta}$ and thus same upper bound holds for $\sup_{\eta \in \mathcal{S}_{k,C_v}(A)} \inf_{\eta_n \in \mathcal{S}_{k,C_v}(A_n)} d_{LP}(\eta, \eta_n)$. By also selecting an arbitrary element of $\mathcal{S}_{k,C_v}(A_n)$ and extending it to an element of $\mathcal{S}_{k,C_v}(A)$, we obtain another upper bound for $\sup_{\eta_n \in \mathcal{S}_{k,C_v}(A_n)} \inf_{\eta \in \mathcal{S}_{k,C_v}(A)} d_{LP}(\eta, \eta_n)$ and thus for $d_H$. The different assumptions come in via different techniques used to bound  $d_{LP}$  by a high probability bound on the $L^2$ norm of the functions in  $F$ and their discretization/extension.



\subsection{Results for graphop neural networks}\label{section:graphopnn_theorems}

Not only are graphops and their discretizations close in $d_M$, but, as we show next, neural networks built from a graphop 
are also close to those built from  graphop discretizations in $d_M$. We iterate that here we are comparing nonlinear operators (graphop neural networks) that are acting on different spaces ($L^2([n] / n)$ for some finite $n$ versus $L^2([0,1])$).

Before stating theoretical guarantees for graphop neural networks, let us introduce some assumptions on the neural network activation function and parameters: 
\begin{assumption}\label{assumption:lipschitz_activation}
  Let the activation function $\rho: \R \to \R$ in the definition of graphop neural networks be $1$-Lipschitz. Let the convolution parameters $h$ be normalized such that  $|h| \leq 1$ element-wise.
\end{assumption}


\begin{theorem}[Graphop neural network discretization]\label{theorem:graphopnn_discretization}
  Let  $A: \mathcal{F} \to \mathcal{F}$. Assume that  $A$ satisfies Assumption \ref{assumption:lipschitz_map} with constant $C_A$ and Assumption \ref{assumption:constant_to_constant} or \ref{assumption:lipschitz_to_lipschitz} with resolutions in $\mathcal{N}$. Fix $n \in \mathcal{N}$ and consider $(k,C_v)$-profiles. Under Assumption \ref{assumption:lipschitz_activation}, we have:
  \begin{align}
  d_M(\Phi(h,A,\cdot), \Phi(h,A_n,\cdot)) \leq& P_1\sqrt{\frac{\overline{C}_A C_v}{n}} + \frac{C_v + 1}{n},\label{eqn:bound_graphopnn_approx}
  \end{align}
where $\overline{C}_A := (n_{\max}\sum_{i = 1}^K C^{i}_A)^L,$ $n_{\max} = \max_{l \in [L]} n_l$, and $P_1$ is a constant depending on $K,L$.
  
 Furthermore, we can invoke the triangle inequality to compare outputs of graphop neural networks built from two different discretizations of  $A$. For any $m,n \in \mathcal{N}$,
  \begin{align}
    d_M(\Phi(h,A_m,\cdot), \Phi(h,A_n,\cdot)) \leq& P_1\sqrt{\overline{C}_A C_v}\left( m^{-\frac{1}{2}} + n^{-\frac{1}{2}} \right)  +  (C_v + 1) \left( n^{-1} + m^{-1} \right).\label{eqn:bound_graphopnn_transfer}
  \end{align}

\end{theorem}

Compared to the main theorems of \cite{ruiz_transferability}, there are two main differences in our results. First, our rate of $O(n^{-1 / 2})$ is slower than the rate of $O(n^{-1})$ in \cite{ruiz_transferability} as a function of $n$. 
Yet, second, their bounds contain a small term that is independent of $n$ and does not go to  $0$ as  $n$ goes to infinity.  
 This small term depends on the variability of small eigenvalues in the spectral decomposition of the convolution operator associated with a graphon. 
 The bound in Theorem~\ref{theorem:graphopnn_discretization}, in contrast, goes to zero.


The proof for this theorem is in Appendix \ref{appendix:proofs_from_graphopnn} for a more general Theorem \ref{theorem:general_graphopnn}. Note that it does not suffice to simply use the fact that the assumptions 
play well with composition with Lipschitz function $\rho$, which would result in a bound involving $\Phi(h,A,\cdot)$ and its discretization $(\Phi(h,A,\cdot))_n$ as a nonlinear operator, as opposed to a bound between $\Phi(h,A,\cdot)$ and $\Phi(h,A_n,\cdot)$.
Our proof shares the same structure as that of Theorem \ref{theorem:approximation_via_discretization} while making sure that the mismatch from discretizing/extending operators does not blow up with composition.

\subsection{Assumptions}

We state and discuss the main assumptions of our $P$-operators, which are not necessarily linear. 

  \begin{assumption}[Lipschitz mapping]\label{assumption:lipschitz_map}
  An operator $A:  \mathcal{F} \to \mathcal{F}$ is \emph{$C_A$-Lipschitz} if  $\|Af - Ag\|_2 \leq C_A \|f - g\|_2$ for any $f,g \in \mathcal{F}$. 
\end{assumption}
  We have already had a finite bound on the operator norm in the definition of $P$-operators. For linear operators, Assumption~\ref{assumption:lipschitz_map} is equivalent to a bounded operator norm and is thus automatically satisfied by linear $P$-operators.

  The next few assumptions are alternatives; only one needs be satisfied by our $P$-operators. Intuitively, they are regularity assumptions that ensure the images of our operator are not too wild and are specifically designed for our discretizations scheme: 
  \begin{subassumption}{assumption}
    \begin{assumption}[Maps constant pieces to constant pieces]\label{assumption:constant_to_constant}
      We say that an operator $A: \mathcal{F} \to \mathcal{F}$ \emph{maps constant pieces to constant pieces} at resolutions in $\mathcal{N} \subset \mathbb{N}$ if for any $n \in \mathcal{N}$, and for any $f \in \mathcal{F}_{[-1,1]}$ that is a.e. constant on each interval  $(u - 1 / n, u]$ for $u \in [n] / n$, $Af$ is also constant on $(u - 1 / n, u]$ for each $u$. 
    \end{assumption}
    \begin{assumption}[Maps Lipschitz functions to Lipschitz functions]\label{assumption:lipschitz_to_lipschitz}
      We say that an operator $A: \mathcal{F} \to \mathcal{F}$ \emph{maps Lipschitz functions to Lipschitz functions with high probability} at resolutions in $\mathcal{N} \subset \mathbb{N}$ if for any $n \in \mathcal{N}$, and for any $f \in \mathcal{F}_{\text{reg}(C_v)}$, $Af$ is  $C_v$-Lipschitz. 
    \end{assumption}
  \end{subassumption}
    This is so far the most restrictive assumption. However, the next lemma describes some dense, sparse and intermediate graphs that satisfy these assumptions.

\begin{lemma}[Well-behaved operators]\label{lemma:constant_to_constant_graphs}
   The following examples satisy our assumptions:
\begin{enumerate}[topsep=0pt,itemsep=0ex,partopsep=1ex,parsep=1ex,leftmargin=0.5cm]
   \item \emph{Bounded-degree graphings:} Let $G$ be a graphing corresponding to the Cayley graph of $\mathbb{Z}$ (two-way infinite paths) or high-dimensional generalizations (infinite 2D and 3D grids). For each $N \in \mathbb{N}$, there exists a locally equivalent graphing $G'_N$ such that its adjacency operator satisfies Assumption \ref{assumption:constant_to_constant} with resolution set $\{x \in \mathbb{N} : x \mid N\}$. 

   \item \emph{Lipschitz graphons:} Let $W$ be a $C_v$-Lipschitz graphon on $\mathcal{F}_{\text{reg}(C_v)}$. Then the Hilbert-Schmidt operator $f \mapsto \int_0^1 W(\cdot,y) g(y) \D y$ satisfies Assumption \ref{assumption:lipschitz_to_lipschitz} with resolution set $\mathbb{N}$. 

   \item \emph{Intermediate graphs:} Let $G$ be a (potentially infinite) graph with a coloring $C: V(G) \to [N]$ for some $N$ such that for each vertex $u,v$ with the same color, the multisets of their neighbors' colors $\{C(u') : (u',u) \in E\}$ are the same. Then its adjacency operator satisfies Assumption \ref{assumption:constant_to_constant} with resolution $N$. An $N$-d hypercube (more generally, Hamming graphs) which is neither bounded-degree nor dense, satisfies the above condition with resolutions in $\{2^{n}\}_{n \in [N]}$.
\end{enumerate}
\end{lemma}

 All our results also hold with a less restrictive assumption that allows for a  failure of Assumption \ref{assumption:constant_to_constant} and \ref{assumption:lipschitz_to_lipschitz} in a small set (see Assumption \ref{assumption:constant_to_constant_whp} and \ref{assumption:lipschitz_to_lipschitz_whp} in the Appendix). The most general results are proven in Appendix  \ref{appendix:proofs_from_section_discretization} and hold in even slightly more relaxed conditions which require the operators to map constant pieces to \emph{Lipschitz} pieces (Assumption \ref{assumption:constant_to_lipschitz}, \ref{assumption:constant_to_lipschitz_whp} in Appendix \ref{appendix:proofs_from_section_discretization}).
\subsection{Deviations and Justifications}\label{section:deviation}

All our theorems hold in slightly modified settings than those by \citet{backhausz_szegedy_2022}. Namely, we allowed for nonlinear $P$-operators, assumed that they have finite $\|\cdot\|_{2 \to 2}$ norm, and used $(k,L)$-profiles where we focus on Lipschitz functions (while \citet{backhausz_szegedy_2022} consider all measurable functions in their profiles). Therefore, we need to ensure that our changes still give us a useful mode of convergence that generalizes dense and sparse graph convergences. 

First, without the linearity assumption, the convergence proof by \citet{backhausz_szegedy_2022} does not hold: we do not know if all limits of nonlinear graphops are still graphops. However, our approximation results (Theorem \ref{theorem:approximation_via_discretization}) show special convergent sequences of nonlinear operators, which go beyond the settings in \citep{backhausz_szegedy_2022}. Studying special nonlinear operator sequences is interesting since graphop NNs themselves can be viewed as nonlinear operators. We also assert that our restriction to operators acting on $L^2$ spaces does not affect convergence guarantees (Theorem 2.14 in \citep{backhausz_szegedy_2022}). 

Next, we show that our restriction to Lipschitz profiles, which is necessary for our proof technique, does not affect convergence either, if we allow our Lipschitz constant to grow with the sequence: 
\begin{theorem}[Growing profiles]\label{theorem:growing_profiles}
  Let $L: \mathbb{N} \to \R$ be a strictly increasing sequence such that $L(n) \xrightarrow{n \to \infty} \infty$. Consider a sequence of $P$-operators  $(A_n: \mathcal{F}_n \to \mathcal{F}_n)_{n \in \mathbb{N}}$ that is Cauchy in the sense that $d_{M}(A_n, A_m) = \sum_{k = 1}^\infty 2^{-k} d_H(\mathcal{S}_{k,L(n)}(A_n), \mathcal{S}_{k,L(m)}(A_m))$ becomes arbitrarily small as $m,n \to \infty$. Then  $(A_n)_{n \in \mathbb{N}}$ converges to the same limit as action convergence.
\end{theorem}
This theorem allows us to replace the $C_v$ constant in our bound with an extremely slowly growing function in  $n$ and get back action convergence as described in \cite{backhausz_szegedy_2022} without any realistic slowdown in the bound. 
\paragraph{Proof sketch} First, by the completeness of the Hausdorff metric over the closed subsets of $\mathcal{P}(\R^{2k})$ - set of probability measures supported on $\R^{2k}$, for any $k$, the statement is equivalent to showing $d_H(\mathcal{S}_k(A), \mathcal{S}_{k,L(n)}(A_n)) \to 0$ as $n \to \infty$. The proof uses a Lipschitz mollification argument to smooth out an arbitrary measurable functions $f_1,\ldots,f_k$ that witness a measure in the $k$-profile of  $A$ for some $k \in \mathbb{N}$. By selecting a Lipschitz mollifier $\phi$, we ensure that convolving  $f_j$ with  $\phi_\eps: x \mapsto \eps^{-1}\phi(x \eps^{-1})$ results in a Lipschitz function that converges to $f$ in  $L^2$ as  $\eps$ goes to  $0$. 



\section{Discussion and Future directions}



In this paper, we study size transferability of finite GNNs on graphs that are discretizations of graphop, a recent notion of graph limit introduced by \citet{backhausz_szegedy_2022}. We achieve this by viewing GNNs as operators that transform one graph signal into another. Under regularity assumptions, we proved that two GNNs, using two different-resolution GSOs discretized from the same graphop, are close in an operator metric built from weak convergence of measures. 

For future direction, a principled study of spectral properties of graphops and graphop neural networks would open doors for techniques from Fourier analysis as used in \citep{ruiz_transferability,cnn_transferability}. This leads to distinct challenges, e.g., the spectral gap is not continuous with respect to local-global limits and thus action convergence, but many more properties of spectral measures of bounded-degree graphs are recently studied \citep{virag_bounded_degree}. 


\bibliographystyle{plainnat.bst}
\bibliography{graphop.bib}

\newpage
\appendix

\section{Additional notation and assumptions}\label{appendix:additional_notations}

\subsection{Notation}
We will use the following extra notations in the proof:
\begin{enumerate}
  \item For some $P$-operator  $A \in L^2(\Omega) \to L^2(\Omega)$, $k \in \mathbb{N}$ and $\{f_1,\ldots,f_k\} =: F \subset L^2(\Omega)$, let $F_A$ be the ordered $2k$-tuple $(f_1,\ldots,f_k, Af_1,\ldots, Af_k)$ and denote $$F_A(x) := (f_1(x),\ldots,f_k(x),Af_1(x),\ldots,Af_k(x)) \in \R^{2k.}$$
  \item For some $P$-operator  $A \in L^2(\Omega) \to L^2(\Omega)$, $k \in \mathbb{N}$ and $\{f_1,\ldots,f_k\} =: F \subset L^2(\Omega)$, let $\mathcal{D}_{A}(F)$ be the entry distribution $\mathcal{D}(f_1, \ldots, f_k, Af_1,\ldots, Af_k)$ .
\end{enumerate}

\paragraph{Cut norm and cut metric}
Here we define concretely the cut metric over the space of graphons. Recall that graphons are $L^{1}([0,1], \mathcal{F}, \lambda)$ Lebesgue-integrable functions. The space of graphons is equipped with a norm known as the \emph{cut norm}:
\begin{equation}
  \|W\|_{\square} := \sup_{S, T \in \mathcal{F}} \left|\int_{S \times T} W(x,y) \D \lambda(x) \D \lambda(y)\right|,
\end{equation}
and a metric known as the \emph{cut metric}:
\begin{equation}
  d_{\square}(W_1,W_2) := \inf_{\phi}  \|W_1 - W_2 \circ \phi\|_{\square},
\end{equation}
where $\phi$ is taken over all measure-preserving bijections from  $[0,1]$ to  $[0,1]$. Intuitively, taking the $\inf$ over all measure-preserving bijections allows the cut metric to identify graphons that are just a rearrangement away from another. This generalizes symmetries in graphs, where permuting the vertices (and the corresponding edges) do not change the graph itself. \citet{lovasz_graph_limit} shows that graphon convergence under the cut metric is well-behaved: every graphon is a limit of a convergent sequence of graphons; and every Cauchy sequences converges to a graphon. This mode of convergence is known as dense graph convergence. 

\paragraph{L\'evy-Prokhorov metric}
The definition of L\'evy-Prokhorov metric on $\mathcal{P}(\mathbb{R}^{2k})$ is:
\begin{align*}
  d_{LP}(\eta_1, \eta_2) := \inf \{\eps > 0 : \eta_1(U) \leq \eta_2(U^{\eps})+ \eps 
  \wedge \eta_2(U) \leq \eta_1(U^{\eps}) + \eps, \forall U \in \mathcal{B}_{2k}\},
\end{align*} 
where $\mathcal{B}_{2k}$ is the Borel $\sigma$-algebra generated from open subsets of  $\R^{2k}$ and $U^\eps := \{ y : \exists x \in  \mathbb{R}^{2k}\|x - y\|_2 < \eps\}$.

\subsection{Assumptions}
As mentioned in the main text, we will work with the following slightly less restrictive set of Assumptions.
\begin{subassumption}{assumption}
        \begin{assumption}[Maps constant pieces to constant pieces with high probability]\label{assumption:constant_to_constant_whp}
      We say that an operator $A: \mathcal{F} \to \mathcal{F}$ \emph{maps constant pieces to constant pieces whp} at resolutions in $\mathcal{N} \subset \mathbb{N}$ if there exists a set $E \subset [0,1]$ with Lebesgue measure $\lambda(E) < \inf_{n \in \mathcal{N}} 1 / n$, such that for any $n \in \mathcal{N}$, and for any $f \in \mathcal{F}_{[-1,1]}$ that is a.e. constant on each interval  $(u - 1 / n, u]$ for $u \in [n] / n$, $Af$ is constant on $(u - 1 / n, u] \backslash E$ and  $\|Af\mathbbm{1}_E\|_1 < \inf_{n \in \mathcal{N}}\frac{1}{n}$. 
    \end{assumption}
    \begin{assumption}[Maps Lipschitz functions to Lipschitz functions with high probability]\label{assumption:lipschitz_to_lipschitz_whp}
      We say that an operator $A: \mathcal{F} \to \mathcal{F}$ \emph{maps Lipschitz functions to Lipschitz functions whp} at resolutions in $\mathcal{N} \subset \mathbb{N}$ if there exists a set $E \subset [0,1]$ with $\lambda(E) < \inf_{n \in \mathcal{N}} 1 / n$, such that for any $n \in \mathcal{N}$, and for any $f \in \mathcal{F}_{[-1,1]}$, $Af$ is  $C_v$ Lipschitz on  $[0,1] \backslash E$ and $\|Af\mathbbm{1}_E\|_1 < \inf_{n \in \mathcal{N}}\frac{1}{n}$. 
    \end{assumption}
\end{subassumption}
\begin{subassumption}{assumption}
    \begin{assumption}[Maps constant pieces to Lipschitz pieces]\label{assumption:constant_to_lipschitz}
      We say that an operator $A: \mathcal{F} \to \mathcal{F}$ \emph{maps constant pieces to Lipschitz pieces} at resolutions in $\mathcal{N} \subset \mathbb{N}$ and constant $C$ if for any $n \in \mathcal{N}$, and for any $f \in \mathcal{F}_{[-1,1]}$ that is a.e. constant on each interval  $(u - 1 / n, u]$ for $u \in [n] / n$, we have that $Af$ is $C$-Lipschitz on each $(u - 1 / n, u]$, for all $u \in  [n] / n$.
    \end{assumption}
    \begin{assumption}[Maps constant pieces to Lipschitz pieces with high probability]\label{assumption:constant_to_lipschitz_whp}
      We say that an operator $A: \mathcal{F} \to \mathcal{F}$ \emph{maps constant pieces to Lipschitz pieces whp} at resolutions in $\mathcal{N} \subset \mathbb{N}$ and constant $C$ if for any $n \in \mathcal{N}$, there exists a set $E \subset [0,1]$ with  $\lambda(E) < \frac{1}{n}$ such that for any $f \in \mathcal{F}_{[-1,1]}$ that is a.e. constant on each interval  $(u - 1 / n, u]$ for $u \in [n] / n$, it holds that $Af$ is $C$-Lipschitz on each $(u - 1 / n, u] \backslash E$, for all $u \in  [n] / n$ and $\|Af\mathbbm{1}_E\|_1 < \inf_{n \in \mathcal{N}}\frac{1}{n}$.
    \end{assumption}
  \end{subassumption}

%
%

\section{Omitted proofs from Section \ref{section:graphopnn}}\label{appendix:proof_from_section_graphopnn}

%
%
%

\begin{proof}[Proof of Lemma \ref{lemma:discretization_of_selfadjoint}]
  Fix $m \in \mathbb{N}$ and $f,g \in \mathcal{F}_m$. Since $P$-operators are bounded, to show that they are self-adjoint, it suffices to show that  $\ip{A_m f}{g} = \ip{f}{A_m g}$ where $\ip{\cdot}{\cdot}$ is the usual inner product in the Hilbert space $\mathcal{F}_m$. We have:
  \begin{align}
    \ip{A_m f}{g} &= \sum_{u \in 1 / m[m]} (A_m f)(u) g(u)\\
                  &= \sum_{u \in  1 / m[m]} \int_{u - \frac{1}{m}}^u Af' \D\lambda \cdot g(u)\\
                  &= \sum_{u \in  1 / m[m]} \int_{u - \frac{1}{m}}^u (Af')g' \D\lambda = \int_0^1 (Af')g' \D\lambda\\
                  &= \int_0^1 f' (Ag') \D\lambda\\
                  &= \sum_{u \in 1 / m [m]} \int_{u - \frac{1}{m}}^u f'(Ag') \D \lambda = \ip{f}{Ag},
  \end{align}
  where the first line is the definition of the inner product in $\mathcal{F}_m$, the second line is the definition of the discretization $A_m$ (recall that  for  $f \in \mathcal{F}_m$, $f' \in \mathcal{F}$ is the extension of  $f$ defined as $f'(x) = f(\left\lceil xm \right\rceil  / m)$), the third line is because $g'$ is constant on fixed  $[u - 1 / m, u]$ intervals for each  $u$ and the fourth line is because $A$ is self-adjoint.
\end{proof}

\section{Theory of graphings}\label{appendix:graphings}
In this subsection, we highlight definitions and key characteristics of graphings so that the paper is self-contained. A much more in-depth discussion can be found in \cite{lovasz_graph_limit}.

\begin{definition}[Borel graphs]
  Let $\Omega$ be a topological space and $(\Omega, \mathcal{F})$ be the corresponding Borel space. A \emph{Borel graph} is a graph $(\Omega, E)$ such that $E \in \mathcal{F} \times \mathcal{F}$.
\end{definition}

The following proposition asserts that bounded-degree graphs without automorphisms are always Borel:
\begin{lemma}[Proposition 18.6 from \cite{lovasz_graph_limit}]
  If $G$ is a bounded-degree graph without automorphisms then there exists a topology $\tau$ on  $V(G)$ (called the \emph{local topology}) such that $G$ is Borel with respect to the Borel space built from $\tau$. 
\end{lemma}
When the graph does have automorphisms, one can break the symmetries by coloring the nodes with some set of colors (the fact that the graph has bounded degree means that one only needs finitely many colors).

We next introduce the main object of interest:
\begin{definition}[Graphings]
  A \emph{graphing} is a quadruple $G = (\Omega, \mathcal{F}, \lambda, E)$ such that $\mathcal{F}$ is a Borel $\sigma$-algebra that makes  $(\Omega, E)$ a Borel graph, and  $\lambda$ is a probability measure on  $(\Omega, \mathcal{F})$ satisfying: for any $A, B \in \mathcal{F}$,
  \begin{equation}
    \int_A \text{deg}_B(x) \D \lambda(x) = \int_B \text{deg}_A (x) \D \lambda(x),
  \end{equation}
  where $\text{deg}_A(x)$ counts the number of neighbors in  $A$ of  $x$.
\end{definition}

As an example, we now describe paths and cycles in terms of graphings: let $V$ be  $[0,1]$, $\mathcal{F}$ the Borel $\sigma$-algebra generated by open intervals with rational endpoints and for each  $x \in [0,1]$, put $(x, x \pm a)$ in $E$ if $x \pm a \in [0,1]$ for some real number $a < 1$. For $a < \frac{1}{2}$, each connected component of the graphing is a finite path. If we consider the edge set $(x, x \pm a \mod 1)$ for rational  $a$, then it is not hard to see that  connected components of the graphing are finite cycles. If $a$ is irrational, then  the resulting graphing is a two-way infinite path, i.e., a path with no ``beginning'' and no ``end''. A formal argument, with an appropriate metric on the space of graphings, can be made to show that the limit of cycles and paths coincides to be the two-way infinite path.  

Graphings are not unique in representing certain graphs. There are weak equivalences between pairs of graphings that are formalized through the notion of local isomorphisms.
\begin{definition}[Local isomorphisms of graphings]
  Let $G_1$,  $G_2$ be two graphings. A measure-preserving map $\varphi: V(G_1) \to V(G_2)$ is a \emph{local isomorphism} if its restriction to almost every connected component of $G_1$ (outside of a set of connected component of measure $0$) is a graph isomorphism with one of the connect components of  $G_2$. More formally,
  \begin{equation}
    \Pr_{\lambda(G_1)}((G_1)_x \equiv (G_2)_{\varphi(x)}) = 1,
  \end{equation}
  where $\equiv$ is rooted graph isomorphism and  $(G_1)_x$ is the connected component of  $G_1$ rooted at  $x$.
\end{definition}

Note that local isomorphism of graphings are not symmetric and the map $\varphi$ needs not be invertible. A stronger notion of equivalence, which is symmetric and transitive is local equivalence:
\begin{definition}[Local equivalence (informal)]\label{def:local-equiv}
  $G_1$ and  $G_2$ are locally equivalent if they have the same subgraph densities $t^*(F,G_1) = t^*(F,G_2)$ for every connected simple graph $F$. 
\end{definition}
The above definition is informal since we have not defined subgraph densities (which is done via the Benjamini-Schramm interpretation of graphings). We state Definition~\ref{def:local-equiv} for readers who are familiar with dense graph convergence of graphons since this is how such convergences are defined. In fact, we can conveniently bypass formally defining local equivalence by the following characterization:
\begin{lemma}[Bi-local isomorphism (Theorem 18.59 in \cite{lovasz_graph_limit})]\label{lemma:bilocal_isomorphism}
  Two graphings are locally equivalent iff there is a third graphing with a local isomorphism to each of them - a property called \emph{bi-local isomorphism}.
\end{lemma}

\section{Omitted proofs from Section \ref{section:discretization}}\label{appendix:proofs_from_section_discretization}
\subsection{Proof of Lemma \ref{lemma:constant_to_constant_graphs}}

We prove each bullet point in  Lemma \ref{lemma:constant_to_constant_graphs} separately, in the next three lemmas. 

In the following lemma, we show that graphings corresponding to infinite paths and high dimensional grids satisfy the assumptions in the main results. Similar to how isomorphic graphs represent the same graph, we only need to specify a locally isomorphic graphing that represent the equivalence class of graphings containing the infinite path and high dimensional grids.
\begin{lemma}[Well-behaved GSOs - Graphings]\label{lemma:constant_to_constant_graphs_graphings}
  Let $G$ be a graphing corresponding to the Cayley graph of $\mathbb{Z}$ (two-way infinite paths) or high-dimensional generalizations (infinite 2D and 3D grids). For a graphing $H$, let  $A(H)$ be its adjacency operator. If  $H$ is regular, let  $\text{deg}(H)$ be the degree of any of its vertices. For each $N \in \mathbb{N}$, there exists locally equivalent graphings:
  \begin{enumerate}
    \item $G'_N$ such that $A(G'_N)$ satisfies Assumption \ref{assumption:constant_to_constant} with resolution set $d(N) = \{x \in \mathbb{N} : x = \alpha N \text{ for some }\alpha \in \mathbb{N}\}$. 
    \item $G''_N$ such that $A(G''_N)$ satisfies Assumption \ref{assumption:constant_to_constant_whp} with resolution set $[N]$.
    \item $G'''_N$ such that $A(G'''_N) / \text{deg}(G)$ satisfies Assumption \ref{assumption:lipschitz_to_lipschitz_whp} with resolution set $[N]$. 
  \end{enumerate}
\end{lemma}
\begin{proof}[Proof of Lemma \ref{lemma:constant_to_constant_graphs_graphings}]
  We will first show the results for two-way infinite paths. Higher dimensional versions follow almost verbatim. Fix $a \in \R \backslash \mathbb{Q}$ irrational. Recall from \cite{lovasz_graph_limit} that the graphing $G = ([0,1],\mathcal{F},\lambda,E)$ where $\mathcal{F}$ is the Borel $\sigma$-algebra generated by open intervals with rational endpoints, $\lambda$ is some probability measure on  $([0,1],\mathcal{F})$ and $E \in \mathcal{F} \times \mathcal{F}$ is defined as:
  \begin{equation}
    E := \{(x, x \pm a \mod 1) \mid x \in [0,1]\}.
  \end{equation}

  That $G$ is a graphing and each of its connected components is a copy of the Cayley graph of $\mathbb{Z}$ generated by $\{-1,1\}$ is asserted in \cite{lovasz_graph_limit}. 
  \begin{enumerate}
    \item Fix $N \in \mathbb{N}$, the goal is to define a graphing $G'_N$ that is locally equivalent to  $G$ such that  $A(G'_N)$ satisfies Assumption \ref{assumption:constant_to_constant} with resolution set  $d(N)$ - the divisor set of  $N$.

      \paragraph{Defining $G'_N$.} Let  $G'_N := ([0,1], \mathcal{F}, \lambda, E'_N)$ where,
      \begin{equation}
        E'_N := \left\{\left(x_j, j - \frac{1}{N} + \left(x_j \pm \frac{a}{N} \mod \frac{1}{N}\right)\right) : j \in [N] / N, x_j \in \left[j - \frac{1}{N}, j\right)\right\}.
      \end{equation}      

      Intuitively, $G'_N$ consists of $N$ disjoint copies of $G$ shrunk to the space  $[0, 1 /N)$. Since $E \subset [0,1] \times [0,1]$,  $G'_N$ is a graphing in the same Borel space as  $G$. 

      \paragraph{$G'_N$ is locally equivalent to  $G$.} From Lemma \ref{lemma:bilocal_isomorphism} it suffices to display a local isomorphism from a third graphing to each of them. Let $G'_{2N}$, defined similarly as $G'_N$, be the third graphing. We claim that $\varphi_1: x \mapsto  2x  \mod 1$ is a local isomorphism from $G'_{2N}$ to $G'_N$ and  $\varphi_2: x \mapsto 2Nx \mod 1$ is a local isomorphism from $G'_{2N}$ to  $G$. Intuitively, $G'_{2N}$ contains $2N$ copies of  $G$ while  $G'_N$ contains  $N$ copies of  $G$. Thus, our local isomorphisms only need to make sure that a connected component in one copy of $G$ in $G'_{N}$ is mapped bijectively to a connected component in another copy of $G$ in $G'_{2N}$. We give a rigorous argument below.

      For  $x$ picked randomly according to  $\lambda$, let  $j \in [2N] / (2N)$ be such that  $x \in [j - 1 / (2N), j)$. By definition of $G'_{2N}$, the connected component  $(G'_{2N})_x$ consists of vertices of the form 
      $$v(k) = j - \frac{1}{2N} + \left( x \pm ak / (2N) \mod 1 / (2N) \right),$$
    for some $k \in \mathbb{Z}$ and there is an edge from $v(k)$ to  $v(k \pm 1)$. Consider  $\varphi_1(x) = 2x \mod 1$ that sends  $v(k)$ to  
      $$\varphi_1(v(k)) := \left(2j - \frac{1}{N} + \left( 2x \pm \frac{ak}{N} \mod \frac{1}{N} \right)\right) \mod 1. $$

      Now consider the connected component of $\varphi_1(x) = 2x \mod 1$ in  $G'_N$. Since $x \in [j - \frac{1}{2N}, j), \varphi_1(x) \in [j' - 1 /N, j')$ where $[N] / N \ni j' := 2j \mod \frac{1}{N}$. Thus, the connected component of $(G'_{N})_{\varphi_1(x)}$ consists of vertices of the form:
      \begin{align}
        v'(k) &:= j' - \frac{1}{N} + \left( \varphi_1(x) \pm ak / N \mod 1 / N \right) \\
              &= \left(2j \mod \frac{1}{N}\right) - \frac{1}{N} + \left( (2x \mod 1) \pm \frac{ak}{N} \mod \frac{1}{N} \right).
      \end{align}

      With some modulo arithmetic manipulation, it is not hard to see that $\varphi_1(v(k)) = v'(k)$ for all $k \in \mathbb{Z}$. By definition of $G'_N$, there is an edge  $(v'(k), v'(k \pm 1))$. Therefore, there is an edge  $(\varphi_1(v(k)), \varphi_1(v(k \pm 1)))$. If $\lambda$ is the uniform measure then we can ignore vertices at the endpoints of our intervals ( points in $[N] / N$ ) and conclude that $(G'_{2N})_x \equiv (G'_{N})_{\varphi_1(x)}$ with probability $1$ when  $x \sim \lambda$.

      Now consider $\varphi_2(x) = 2Nx \mod 1$ that sends  $v(k)$ to:
       \begin{equation}
        \varphi_2(v(k)) := \left(2Nj - 1 + \left( 2Nx \pm ak \mod 1 \right)\right) \mod 1.
      \end{equation}

      The connected component of $\phi_2(x)$ in  $G$ is in  $[0,1]$  and consists of vertices of the form:
      \begin{align}
        v''(k) &:= (2Nj - 1 \mod 1) + (2Nx \pm ak \mod 1). 
      \end{align}
      Thus $v''(k) = \varphi_2(v(k))$ for each  $k \in \mathbb{Z}$ and the definition of $G$ implies that  $(G'_{2N})_x \equiv G_{\varphi_2(x)}$ since both are isomorphic to the two-way infinite path. If  $\lambda$ is the uniform measure then the above hold for a.e.  $x \in [0,1]$. 

      \paragraph{$G'_N$ satisfies Assumption \ref{assumption:constant_to_constant} with resolution set  $d(N)$.} Let $D$ be a divisor of  $N$, then $D \alpha = N$ for some  $\alpha \in \mathbb{N}$. The intuition is rather straightforward, since $D$ is a divisor of  $N$, the partition into  $N$ equal intervals of $[0,1]$ is simply a finer partition into  $D$ equal intervals. By construction, a vertex $x$ of $G'_N$  only has neighbors in the $1 / N$ interval containing it. Therefore,  if $f$ is constant on each of the  $D$ pieces, it is also constant on each of the $N$ pieces and  $f(y) = f(x) = f(x') = f(y')$ for any $(x,y), (x',y')$ neighbors such that  $x$ and  $x'$ come from the same  $1 / D$ pieces. 

      Here is a more formal argument. Fix $f \in \mathcal{F}_{[-1,1]}$ that is a.e. constant on each interval $(u - 1 / d, u]$ for  $u \in [D] / D$. We need to show that $A(G'_N)f$ is a.e. constant on the same pieces where  $A(H)$ is the adjacency operator of a regular graphing  $H$. Fix $d \in [D]$, fix $x,x' \in [d - 1 / D, d)$. Let $j$ be such that  $x \in [ j - 1 / N, j)$ and $j'$ be such that  $x' \in  [j' - 1 / N, j')$. Since $D \alpha = N$, we have  $[j - 1 / N, j), [j' - 1 / N, j') \subseteq [d - 1 /D, d)$. Furthermore, by definition of $G'_N$, all neighbors of  $x$ and  $x'$ are in  $[j - 1 / N, j)$ and  $[j' - 1 / N, j')$ respectively, and thus in  $[d - 1 / D, d)$. Therefore,  $A(G'_N)f(x) = 2 f(x) = 2 f(x') = A(G'_N)f(x')$ for all  $x, x' \in [d - 1 / D, d)$, which means that $A(G'_N)f$ is constant on each of the pieces  $[d - 1 / D, d)$ for each $ d \in  [D]$.

    \item 
      Fix $N \in \mathcal{N}$, the goal is to define a graphing $G''_N$ that is locally equivalent to  $G$ such that  $A(G''_N)$ satisfies Assumption \ref{assumption:constant_to_constant_whp} with resolution set $[N]$. 

      Since $\mathbb{Q}$ is dense in $\mathbb{R}$, there exists a number  $\delta(N) < \frac{1}{4N^2}$ such that $\frac{1}{4N^2} + \delta(N)$ is an irrational number. Let  $G''_N = ([0,1], \mathcal{F}, \lambda, E''_N)$ where,
       \begin{align}
         E''_N := \left\{\left( x, x \pm \left(\frac{1}{4N^2} + \delta(N)\right)\mod 1 \right)  : x \in [0,1]\right\}.
      \end{align}

      That $G''_N$ is locally equivalence to  $G$ is easily seen via the ambiguity of selecting  $a$ when defining  $G$. Now we show that  $A(G''_N)$ satisfies Assumption \ref{assumption:constant_to_constant_whp} with resolution set  $[N]$. Pick $M \leq N$  and $f \in \mathcal{F}_{[-1,1]}$ that is a.e. constant on each pieces $(m -  1 / M, m]$ for each  $m \in  [M] / M$. Let 
      $$E = \bigcup_{m \in [M] / M} \left( m - \frac{1}{4N^2} -  \delta(N), m + \frac{1}{4N^2} + \delta(N) \right] $$
      then 
        $$\lambda(E) = \sum_{m' = 1}^M 2 (1 / (4N^2) +  \delta(N)) \leq \frac{M}{N^2} < \frac{1}{N}.$$

        Pick $x,x' \in (m - 1 / M , m] \backslash E$ then $x, x' \in (m - 1 / M + \eps, m - \eps]$ where $\eps = 1 / (4N^2) + \delta(N)$. Thus we have both $x \pm \eps$ and $x' \pm \eps$ are in  $(m - 1 / M, m]$. By definition of  $G''_N$, all neighbors of  $x$ and  $x'$ are in the same  $1 / M$ piece as  $x$ and  $x'$. Therefore, if  $f \in \mathcal{F}_{[-1,1]}$ are constant on these pieces, so is $Af$. 

    \item 
      Fix $N \in \mathcal{N}$, the goal is define a graphing $G'''_N$ that is locally equivalent to  $G$ such that  $A(G'''_N)$ satisfies Assumption \ref{assumption:lipschitz_to_lipschitz_whp} with resolution set  $[N]$.

      Since  $\mathbb{Q}$ is dense in  $\mathbb{R}$, there exists a number  $\delta(N) < 1 / (4N)$ such that  $1 / (4N) + \delta(N)$ is an irrational number. Let $G'''_N = ([0,1], \mathcal{F}, \lambda, E'''_N)$ where 
       \begin{equation}
         E'''_N := \left\{\left(x, x \pm \left(\frac{1}{4N} + \delta(N)\right) \mod 1\right) : x \in [0,1]\right\}
      \end{equation}

      That $G'''_N$ is locally equivalent to  $G$ is easily seen via the ambiguity of selecting $a$ when defining $G$. Now we show that  $A(G'''_N)$ satisfies Assumption \ref{assumption:lipschitz_to_lipschitz_whp} with resolution set  $[N]$. Pick  $M \leq N$ and  $f \in \mathcal{F}_{\text{reg}(C_v)}$ that is a.e. $C_v$-Lipschitz on $[0,1]$. Set $\eps = \frac{1}{4N} + \delta(N)$ and let:
      \begin{equation}
        E = [0, \eps) \cup [1 - \eps, 1).
      \end{equation}
      Then $\lambda(E) = 2 \eps < \frac{1}{N}$. Pick  $x,x' \in [0,1] \backslash E$, then $x \pm \eps$ and  $x' \pm \eps$ do not `loop over' in the interval  $[0,1]$ ( $y \mod 1 = y$, for  $y \in \{x,x'\} + \{\pm \eps\}$ ). Thus we have:
      \begin{align}
        \left|Af(x) - Af(x')\right| &= \frac{\left|f(x + \eps) + f(x - \eps) - f(x' + \eps) - f(x' - \eps)\right|}{2}\\
                                    &\leq \frac{|f(x+\eps) - f(x' +\eps)| + |f(x - \eps) - f(x' - \eps)|}{2}\\
                                    &\leq C_v |x - x'|,
      \end{align}
where in the last line we use Lipschitz property of $f$. This finishes the proof.
  \end{enumerate}  
\end{proof}

\begin{lemma}[Well-behaved operators - Graphons]\label{lemma:well_behaved_graphons}
    Let $W: [0,1]^2 \to [0,1]$ be a Lipschitz graphon, with Lipschitz constant $C$. In other words, $W$ has finite $L^2$ norm and is Lipschitz in both variables. Then the Hilbert-Schmidt integral operator $H$ that defines the adjacency operator of $W$ satisfies Assumption \ref{assumption:lipschitz_to_lipschitz} at any resolution when applied to graph(on) signal $f$ with $L^1$ norm $1$.
\end{lemma}
\begin{proof}
    The proof follows from definition. Take $f \in L^2([0,1])$ such that $\|f\|_{L^2} = 1$, then:
    \begin{align}
        |Hf(x) -  Hf(y)| &= \left| \int_{0}^1 W(x,z) f(z) \D z - \int_{0}^1 W(y,z) f(z) \D z\right| \\
        &\leq  \int_{0}^1 \left|W(x,z) - W(y,z))\right| \cdot |f(z)| \D z\\
        &\leq \int_{0}^1 C|x - y| \cdot |f(z)| \D z = C|x - y|,
    \end{align}
    where the first line is definition of the Hilbert-Schmidt operator $H$, the second line is triangle inequality and the last line is due to Lipschitzness of $W$ and $L^1$ norm of $f$
\end{proof}

\begin{lemma}[Well-behaved operators - General graphs] \label{lemma:well_behaved_general}
Let $G$ be a (potentially countably infinite) graph with a coloring $C: V(G) \to [N]$ for some $N$ such that for each vertex $u,v$ with the same color, the multisets of their neighbors' colors $\{C(u') : (u',u) \in E\}$ are the same. Additionally, assume that the cardinality of vertices of each color is the same: that there is a bijection from $\{v : C(v) = c\}$ to $\{v : C(v) = c;\}$ for any colors $c,c'$. Then its adjacency operator satisfies Assumption \ref{assumption:constant_to_constant} with resolution $N$. 
\end{lemma}
\begin{proof}
Given that the cardinality of vertices of each color is the same, it is straightforward to map (via a Lebsesgue-measure preserving bijection) vertices of $V$ into equipartition of $[0,1]$ into $N$ pieces $I_1,..., I_N$ such that each partition contains vertices of the same color and vertices from different partitions will have different colors. Let $A$ be the normalized adjacency operator of $G$ and $f$ be a function with finite $L^2$ norm such that $f$ is constant on each $I_j$ for each $j$ from $1$ to $N$. The goal is to show that $Af$ is also constant on these pieces.

We show this by direct computation. Pick a vertex $x$ and another vertex $y$ from the same piece, say $I_k$ for some particular $k$. By our construction, $x$ and $y$ have the same color since they come from the same piece. By our assumption, the multisets of their neighbors' colors is the same. However, since $f$ is constant on each pieces, we have $f(z) = f(u)$ for $u,v$ of the same color. Therefore, the multisets $\{f(z) : (z,x) \in E\}$ and $\{f(z) : (z,y) \in E\}$ is exactly the same. Taking the appropriate countable sum over each multiset thus result in the same number, i.e. $Af(x) = Af(y)$, which finishes the proof. 
\end{proof}
 Note that we can drop the countable requirement of the previous proof by invoking an appropriate integral definition. Now we show that the hypercubes - an intermediate graph that is neither dense nor bounded degree, satisfies the requirements of Lemma \ref{lemma:well_behaved_general}. 

\begin{lemma}[Well-behaved operators - Hypercubes] \label{lemma:well_behaved_hypercubes}
Let $C$ be a hypercube of dimension $N$. Then the normalized adjacency matrix of $C$ satisfies assumption \ref{assumption:constant_to_constant} with resolutions $2^{[n]}$ for each $n < N$.
\end{lemma}
\begin{proof}
    We first display a mapping of the vertices in the hypercube over the interval $[0,1]$ such that Assumption \ref{assumption:constant_to_constant} will be shown to hold. Recall that a hypercube vertices can be represented by a binary string of $N$ numbers. Here, similarly, we will associate a vertex $v$ of a hypercube with a number in $[0,1]$ by adding a $0.$ in front of its binary representation. For example, $0.110101$ (as a binary number) is a vertex correspond to the string $110101$ when $N = 6$. We call this representation mapping $f: V \to [0,1]$. Two vertices $u,v$ are connected by an edge iff $f(v)$ and $f(v)$ differs by exactly one digit in their binary representation. 

    To get a hypercube representation over the interval $[0,1]$, we simply take disconnected copies of (uncountably) infinitely many hypercubes described above. To be more precise, for a number $x \in [0,1]$, let $A(x)$ be the $N$-letter binary string such that $0.A$ (as a string) is the trunction of $x$ to the $N$-th digit after the binary point. Then, $x$ is connected to $y \in [0,1]$ iff $A(x)$ and $A(y)$ differs in exactly one digit and $x - A(x) = y - A(y)$. For example, for $N = 4$, we connect $0.110101$ with $0.010101$ , $0.100101$, $0.111101$ and $0.110001$. Notice how the first four digits after the dots differ from the original number at exactly one letter, and the last few digits always stay the same. In this example, $A(x) = 0.1101$ for $x = 0.110101$.

    With this vertex representation in mind, we verify the conditions of Lemma \ref{lemma:well_behaved_general}. Let $V$. Fix $n < N$ and divide $[0,1]$ into $2^n$ equipartitions and let vertices in the same partition enjoy the same color. We will also number the partition/color by the binary representation of their left hand endpoint. For example, the first interval is labelled $00..0$ ($n$ digits), the second interval is labelled $00..01$ ($n$ digits). Select a vertex $x$ from a partition with label $a_1a_2\ldots a_n$. Then, by the definition of our hypercube representation, the neighbors of $x$ have colors:
    \begin{itemize}
        \item Exactly $1$ neighbor with color $(a_1 + 1 \mod 1) a_2 \ldots a_n$.
        \item Exactly $1$ neighbor with color $a_1(a_2+1 \mod 1)a_3\ldots a_n$.
        \item \ldots
        \item Exactly $1$ neighbor with color $a_1a_2 \ldots a_{n-1}(a_n + 1 \mod 1)$.
        \item Remaining $N - n$ neighbors all have color $a_1a_2\ldots a_n$.
    \end{itemize}

    Therefore, the multiset of neighbors' colors of $x$ only depends on the fact that $x$ comes from a partition with label $a_1 \ldots a_n$ and thus, two vertices of the same color has the same multiset of neighbor's colors. Since the condition of Lemma \ref{lemma:well_behaved_general} holds, we have the conclusion for this particular resolution $n$. Since $n$ was chosen arbitrarily, the result holds for every $n < N$.
\end{proof}
In future work, we formalizes the way in which the above proof holds for limiting object of hypercube.

\subsection{Proof of Theorem \ref{theorem:approximation_via_discretization}}
In this section, we prove a slightly more general version of Theorem \ref{theorem:approximation_via_discretization}.
\begin{theorem}[General approximation theorem]
  Let $A: \mathcal{F} \to \mathcal{F}$ be a $P$-operator satisfying Assumption \ref{assumption:lipschitz_map} with constant $C_A$; Assumption \ref{assumption:constant_to_lipschitz} with constant $C_c$ and resolutions in $\mathcal{N}$. Fix $n \in \mathcal{N}$ and consider $(k, C_v)$-profiles. Let $A_n: \mathcal{F}_n \to \mathcal{F}_n$ be a discretization of  $A$ as defined in Equation \eqref{eqn:discretization}. Then:
  \begin{equation}
    d_M\left(A, A_n\right) \leq 8\left(\sqrt{\frac{C_A C_v}{n}} + \frac{C_v + C_c}{n}\right).
  \end{equation}

  If instead of Assumption \ref{assumption:constant_to_lipschitz}, $A$ satisfies Assumption \ref{assumption:constant_to_lipschitz_whp} with constant  $C_c$ and resolution set  $\mathcal{N}$, then:
  \begin{equation}
    d_M\left(A, A_n\right) \leq 8\left(\sqrt{\frac{C_A C_v + 1}{n}} + \frac{C_v + C_c + 1}{n}\right).
  \end{equation}

  If instead of Assumption \ref{assumption:constant_to_lipschitz} and \ref{assumption:constant_to_lipschitz_whp}, $A$ satisfies Assumption \ref{assumption:lipschitz_to_lipschitz_whp} with resolution set  $\mathcal{N}$ then:
  \begin{equation}
    d_M\left(A, A_n\right) \leq 8\left(\sqrt{\frac{C_A C_v + 1}{n}} + \frac{C_v + 1}{n}\right).  
  \end{equation}
\end{theorem}

\begin{proof}[Proof of Theorem \ref{theorem:approximation_via_discretization}]
  Fix $k \in \mathbb{N}$. In order to derive an upper bound for $d_M$, we find an upper bound for the Hausdorff distance $d_H$ between the two $k$-profiles of  $A$ and  $A_n$. 

  \paragraph{Bounding $\sup_{\eta \in \mathcal{S}_{k,C_v}(A)} \inf_{\eta_n \in \mathcal{S}_{k,C_v}(A_n)} d_{LP}(\eta, \eta_n)$.} To bound the $\sup \inf$ quantity, we first select an arbitrary  $\overline{\eta} \in \mathcal{S}_{k,C_v}(A)$. From this measure, we will construct a measure $\overline{\eta}_n \in \mathcal{S}_{k,C_v}(A_n)$. If we can upper bound $d_{LP}(\overline{\eta}, \overline{\eta}_n ) < M$ then we have:
  \begin{equation}
    \inf_{\eta_n \in \mathcal{S}_{k,C_v}(A_n)} d_{LP}(\overline{\eta},\eta_n) \leq d_{LP}(\overline{\eta},\overline{\eta}_n) < M, \text{ for all } \overline{\eta} \in \mathcal{S}_{k,C_v}(A).
  \end{equation}
  If further $M$ does not depend on the choice of  $\overline{\eta}$, then we have $$\sup_{\eta \in \mathcal{S}_{k,C_v}(A)} \inf_{\eta_n \in \mathcal{S}_{k,C_v}(A_n)} d_{LP}(\eta, \eta_n) \leq M.$$

  We now proceed with this plan. Fix an arbitrary $\overline{\eta} \in \mathcal{S}_k(A)$, by definition of $k$-profiles, there is a corresponding tuple $F = (f_1, \ldots, f_k)$ with elements in $\mathcal{F}_{\text{reg}(C_v)}$ such that $\mathcal{D}_A(F) = \overline{\eta}$. 

  Form:
  \begin{equation}
    F' := \left\{\mathcal{F}_{n, \text{reg}(C_v)} \ni f'_j: u \mapsto n\int_{u - 1 / n}^u f_j \D \lambda \mid j \in [k] \right\}, \qquad \overline{\eta}_n := \mathcal{D}_{A_n}(F').
  \end{equation}
  That $f'_j \in \mathcal{F}_{n, \text{reg}(C_v)}$ is asserted in Lemma \ref{lemma:restriction_of_bounded_range_is_bounded}. 

  \paragraph{Bounding $d_{LP}(\overline{\eta}, \overline{\eta}_n)$}
  For some $\eps$ to be we want to show that $d_{LP}(\overline{\eta}, \overline{\eta}_n) \leq \eps$, which is equivalent to showing, for any $U \in \mathcal{B}_{2k}$, 
  \begin{equation}
    \overline{\eta}(U) \leq \overline{\eta}_n(U^{\eps}) + \eps \qquad \text{and} \qquad \overline{\eta}_n(U) \leq \overline{\eta}(U^{\eps})+\eps.
  \end{equation}

Fix any $U \in \mathcal{B}_{2k}$, we have 
\begin{align}
  \overline{\eta}(U) = \int_{\R^{2k}} \mathbbm{1}_{U} \D \mathcal{D}_A(F) = \int_0^{1} \mathbbm{1}_{F_A \in U} \D \lambda,
\end{align}
and 
\begin{align}
  \overline{\eta}_n(U^{\eps}) = \int_{\R^{2k}} \mathbbm{1}_{U^{\eps}} \D \mathcal{D}_{A_n}(F') &= \sum_{u \in [n] / n} \frac{1}{n}\mathbbm{1}_{F'_{A_n} \in U^{\eps}}. 
\end{align}

Subtracting both sides yield:
\begin{align}  
\overline{\eta}(U) - \overline{\eta}_n(U^{\eps}) =\sum_{u \in [n] / n} \int_{u - 1 /n }^{u} \mathbbm{1}_{F_A(x) \in U}  -\mathbbm{1}_{F'_{A_n}(u) \in U^{\eps}}\D \lambda(x).\label{eqn:xu_yueps}
\end{align}
 
Similarly, we have:
\begin{align}  
\overline{\eta}_n(U) - \overline{\eta}(U^{\eps})= \sum_{u \in [n] / n} \int_{u - 1 / n}^u \mathbbm{1}_{F'_{A_n}(u) \in U} -\mathbbm{1}_{F_A(x) \in U^{\eps}}  \D \lambda(x)\label{eqn:yu_xueps}
\end{align}

Let $y_x = \left\lceil xn \right\rceil / n$. Recall that $E \subset [0,1]$ defined in Assumption \ref{assumption:constant_to_lipschitz_whp} and \ref{assumption:lipschitz_to_lipschitz_whp} is the set of $x$ where Lipschitzness of the image under $A$ may fail. Let  $E = \emptyset$ if we are using Assumption \ref{assumption:constant_to_lipschitz} alternatively. Define the events:
\begin{align}
  \mathcal{E}^1_U(\eps) &= \{x  : F_A(x) \in U \wedge F'_{A_n}(y_x) \not \in U^{\eps} \}\\
  \mathcal{E}^2_{U}(\eps) &= \{x  : F'_{A_n}(y_x) \in U \wedge F_A \not \in U^{\eps} \}\\
  \mathcal{E}'(\eps) &= \{x : \|F_A(x)  - F'_{A_n}(y_x)\|_2 > \eps\}\\
  \mathcal{E}_j(\eps) &= \{x : |f_j(x) - f'_j(y_x)| > \eps / \sqrt{2k} \}, &j = 1..k\\
\mathcal{E}_{j,A}(\eps) &= \{x : |Af_j(x) - A_nf'_j(y_x)| > \eps / \sqrt{2k} \}, &j = 1..k.
\end{align}

Using this notation, one also has
\begin{align}
  \overline{\eta}(U) - \overline{\eta}_n(U^{\eps}) \leq \lambda(\mathcal{E}^1_U) \qquad \text{and} \qquad \overline{\eta}_n(U) - \overline{\eta}(U^{\eps}) \leq \lambda(\mathcal{E}^2_U).
\end{align}

It is straightforward to see $\mathcal{E}^l_U(\eps) \subseteq \mathcal{E}'(\eps), l = 1,2$ for any $U$ by definition of  $U^\eps$. Furthermore,  
\begin{equation}
  \mathcal{E}'(\eps) \subseteq \bigcup_{j = 1}^k \mathcal{E}_j(\eps) \cup \mathcal{E}_{j,A}(\eps),
\end{equation}
since if all $2k$ dimensions are bounded in absolute value by $\eps / \sqrt{2k}$ then the Euclidean distance of the vector is bounded by $\eps$. 

Therefore it suffices to bound $\lambda(\mathcal{E}_j(\eps)) + \lambda(\mathcal{E}_{j,A}(\eps))$ for each $j \in [k]$.

\paragraph{Bounding $\lambda(\mathcal{E}_j(\eps))$.}
Since $f_j$ is  $C_v$-Lipschitz for all $j$, we have:
\begin{align}
  |f_j(x) - f'_j(y_x)| &= \left|f_j(x) - n \int_{y_x - 1 / n}^{y_x} f_j(z) \D\lambda(z)\right|\\
                     &= n\left|\int_{y_x - 1 / n}^{y_x} f_j(x) - f_j(z) \D \lambda(z)\right|\\
                     &= n \int_{y_x - 1 / n}^{y_x}|f_j(x) - f_j(z)| \D \lambda(z)\\
                     &\leq n \int_{y_x - 1 / n}^{y_x} \frac{C_v}{n} \D\lambda(z) = \frac{C_v}{n},
\end{align}
where the first line is definition of $f'_j$, the second line is because  $\lambda((u - 1 / n, u]) = \frac{1}{n}$, the third line is by triangle inequality and the last line is because of Lipschitzness of $f_j$. 

Thus, choosing $\eps > \sqrt{2k} C_v / n$ means that  $\lambda(\mathcal{E}_{j}(\eps)) = 0$. (We can tighten this bound by only assuming that $f_j$ is  $C_v$-Lipschitz outside a set of small measure.)

\paragraph{Bounding $\lambda(\mathcal{E}_{j,A})$.}
Let $\mathcal{F}_{[-1,1]} \ni \tilde{f}$ be the extension of $f'$ defined as  $\tilde{f}(x) = f'(\left\lceil xn \right\rceil / n)$ for all $x \in [0,1]$. Note that $\tilde{f}$ is not continuous in general and hence not Lipschitz. We have for any $x \in [0,1]$:
\begin{align}
  |Af_j(x) - A_n f'_j(y_x)| &= \left| Af_j(x) - n\int_{y_x - \frac{1}{n}}^{y_x} A\tilde{f}_j(z) \D \lambda(z) \right|\\
                            &\leq n \int_{y_x - \frac{1}{n}}^{y_x} \left|Af_j(x) - A\tilde{f}_j(z)\right| \D \lambda(z),
\end{align}
where we used uniformity of $\lambda$ and triangle inequality. From here, we proceed slightly differently depending on the specific assumptions. 

\paragraph{Proof via Assumption \ref{assumption:constant_to_lipschitz} or \ref{assumption:constant_to_lipschitz_whp}.}
If $A$ satisfies Assumption \ref{assumption:constant_to_lipschitz} or Assumption \ref{assumption:constant_to_lipschitz_whp}, then we have for each $x$:
\begin{align}
  |Af_j(x) - A_n f'_j(y_x)| &\leq n \int_{y_x - \frac{1}{n}}^{y_x} \left|Af_j(x) - A \tilde{f}_j(x)\right| + \left|A \tilde{f}_j(x) - A\tilde{f}_j(z)\right| \D \lambda(z)\\
                            &\leq \left|Af_j(x) - A\tilde{f}_j(x)\right| + n \int_{y_x - \frac{1}{n}}^{y_x}\left|A \tilde{f}_j(x) - A\tilde{f}_j(z)\right| \D \lambda(z).
\end{align}

Define the following events:
\begin{align}
  \mathcal{E}_{j,A}^1(\eps) &:= \left\{x : \left|A f_j(x) - A \tilde{f}_j(x) \right| > \frac{\eps}{2\sqrt{2k}}\right\}\\
  \mathcal{E}_{j,A}^2(\eps) &:= \left\{x \not \in E : n \int_{y_x - 1 / n}^{y_x} \left|A \tilde{f}_j (x) - A \tilde{f}_j(z)\right| \D \lambda(z) > \frac{\eps}{2\sqrt{2k}}\right\} 
\end{align}

Then it is clear that $\mathcal{E}_{j,A} \subseteq \mathcal{E}_{j,A}^1 \cup \mathcal{E}_{j,A}^2 \cup E$ and thus $\lambda(\mathcal{E}_{j,A}) \leq \lambda(\mathcal{E}_{j,A}^1) + \lambda(\mathcal{E}_{j,A}^2) + \lambda(E)$.

\paragraph{Bounding $\lambda(\mathcal{E}^1_{j,A}(\eps))$ via Assumption \ref{assumption:constant_to_lipschitz} or \ref{assumption:constant_to_lipschitz_whp}.}
Because of Assumption \ref{assumption:lipschitz_map}, we have:
$\|A\tilde{f}_j - Af_j\|_2 \leq C_A \|\tilde{f}_j - f_j\|_2$. By $L^{p}$ norms inequality, we have 
\begin{equation}
  \|A\tilde{f}_j(x) - Af_j(x)\|_1 \leq C_A \|\tilde{f}_j - f_j\|_2 \leq \frac{C_A C_v}{n}.
\end{equation}
where the last inequality is due to 
\begin{align}
  \|\tilde{f}_j - f_j\|^{2}_2 &= \int_0^{1} (\tilde{f}_j(x) - f_j(x))^{2} \D \lambda(x) \\
                              &= \sum_{u \in [n] / n} \int_{u - 1 / n}^u (f'_j(u / n) - f_j(x))^{2} \D \lambda(x) \\
                              &= \sum_{u \in  [n] / n} \int_{u - 1 / n}^u \left(n\int_{(u - 1) / n}^{u / n} f_j(z) \D \lambda(z) - f_j(x) \right)^2  \D \lambda(x)\\
                              &\leq \sum_{u \in [n] / n} \int_{u - 1 / n}^u n^2 \left(\int_{u - 1 / n}^u |f_j(z) - f_j(x)| \D \lambda(z)\right)^2 \D \lambda(x)\\
                              &\leq \frac{C_v^2}{n^2}.
\end{align}

 We have:
\begin{align}
  \frac{C_AC_v}{n} &\geq \int_0^1 |A \tilde{f}_j(x) - Af_j(x)| dx \\
                   &= \int_{\mathcal{E}^1_{j,A}(\eps)} |A \tilde{f}_j(x) - Af_j(x)| dx + \int_{[0,1] \backslash \mathcal{E}^1_{j,A}(\eps)}|A\tilde{f}_j(x) - Af_j(x)| dx\\
                   &\geq \frac{\eps}{2\sqrt{2k}} \lambda(\mathcal{E}^1_{j,A}(\eps)) + 0.
\end{align}
Thus selecting $\eps > 2\sqrt{(\sqrt{2} k\sqrt{k} C_A C_v) / n} = 2^{\frac{5}{4}} k^{\frac{3}{4}} (C_AC_v / n)^{\frac{1}{2}}$ gives $\lambda(\overline{\mathcal{E}}_{j,A}( \eps)) \leq \frac{\eps}{2k}$.

\paragraph{Bounding $\lambda(\mathcal{E}^2_{j,A}(\eps))$ via Assumption \ref{assumption:constant_to_lipschitz} or \ref{assumption:constant_to_lipschitz_whp}.}
Notice that $\tilde{f}_j$ is constant in each  $(u - 1 / n, u]$ and thus by Assumption \ref{assumption:constant_to_lipschitz} or \ref{assumption:constant_to_lipschitz_whp}, $A \tilde{f}_j$ is  $C_c$-Lipschitz in each  $(u - 1 / n, u] \backslash E$. Therefore we have, for $x \in [0,1]\backslash E$ and $z \in (y_x - 1 / n, y_x] \backslash E$:
\begin{align}
  |A\tilde{f}_j(x) - A\tilde{f}_j(z)|  &\leq C_c |x - z| \leq \frac{C_c}{n}.
\end{align}

When $z \in E$, we use the second condition in Assumption \ref{assumption:constant_to_lipschitz_whp} to get $\|A\tilde{f}_j \mathbbm{1}_E\|_1 \leq \frac{1}{n}$ and thus for any $x \not \in E$:
\begin{align}
  &n \int_{y_x - 1 / n}^{y_x} \left|A \tilde{f}_j(x) - A \tilde{f}_j(z) \right| \D \lambda(z) \\
  =& n \int_{(y_x - 1 / n, y_x] \cap E} \left|A \tilde{f}_j(x) - A \tilde{f}_j(z) \right| \D \lambda(z) +  n \int_{(y_x - 1 / n, y_x] \backslash E} \left|A \tilde{f}_j(x) - A \tilde{f}_j(z) \right| \D \lambda(z)\\
\leq&   \frac{C_c}{n} + n \int_{(y_x - 1 / n, y_x] \cap E} \left|A \tilde{f}_j(x)\right|+\left|A \tilde{f}_j(z) \right| \D \lambda(z)\\
\leq& \frac{C_c}{n} + |A\tilde{f}_j(x)| n \lambda(E_x) + \|A\tilde{f}_j \mathbbm{1}_{E_x}\|_1\\
\leq& \frac{C_c + 1}{n} + |A\tilde{f}_j(x)| n \lambda(E_x),
\end{align}
where $E_x = (y_x - 1 / n, y_x] \cap E$.

For $x \in \mathcal{E}_{j,A}^2(\eps)$, we have: 
\begin{equation}
  \frac{C_c + 1}{n} + |A\tilde{f}_j(x)| n \lambda(E_x) > \frac{\eps}{2\sqrt{2k}},
\end{equation}
or equivalently,
\begin{equation}
  |A\tilde{f}_j(x)| > \frac{1}{n\lambda(E_x)}\left(\frac{\eps}{2\sqrt{2k}} - \frac{C_c + 1}{n}\right)
\end{equation}

Since $1 / n \geq \|A\tilde{f}_j \mathbbm{1}_E\|_1$, we have:
\begin{align}
  1 &\geq \left(\frac{\eps}{2\sqrt{2k}} - \frac{C_c + 1}{n}\right) \int_{\mathcal{E}^2_{j,A}}\frac{1}{\lambda(E_x)} \D \lambda(x) + 0\\
    &\geq \left(\frac{\eps}{2\sqrt{2k}} - \frac{C_c + 1}{n}\right) \sum_{u \in  [n] / n}\frac{1}{\lambda(E_u)} \int_{\mathcal{E}_{j,A}^2 \cap (u - 1 / n, u]} 1 \D \lambda(x)\\
    &\geq \left(\frac{\eps}{2\sqrt{2k}} - \frac{C_c + 1}{n}\right) n \lambda(\mathcal{E}_{j,A}^2).
\end{align}

Thus choosing $\eps > 4\sqrt{2k} (C_c + 1) / n$ means that:
\begin{align}
  1 \geq \frac{\eps}{4\sqrt{2k}} n \lambda(\mathcal{E}_{j,A}^2),
\end{align}
or
\begin{align}
  \lambda(\mathcal{E}_{j,A}^2) \leq \frac{4\sqrt{2k}}{\eps n}. 
\end{align}

Finally, choosing $\eps > \sqrt{\frac{16k\sqrt{2k}}{n}} = \frac{2^{\frac{9}{4}} k^{\frac{3}{4}}}{n^{\frac{1}{2}}}$ makes $\lambda(\mathcal{E}_{j,A}^2) \leq \frac{\eps}{4k}$; and choosing $\eps > \frac{4k}{n}$ makes  $\lambda(E) < \frac{1}{n} < \frac{\eps}{4k}$.

\paragraph{Putting everything together via Assumption \ref{assumption:constant_to_lipschitz} or \ref{assumption:constant_to_lipschitz_whp}.}
Thus, we can choose $\overline{\eps} = 8k\left(\sqrt{\frac{C_AC_v + 1}{n}} + \frac{C_v+ C_c + 1}{n}\right) $ to get:
\begin{equation}
  \lambda(\mathcal{E}'(\eps)) \leq \sum_{j = 1}^k \frac{\overline{\eps}}{k} = \overline{\eps},
\end{equation}
which allows us to conclude:
\begin{equation}
  d_{LP}(\overline{\eta}, \overline{\eta}_n) \leq \overline{\eps}.
\end{equation}

Since $\overline{\eta}$ was chosen arbitrarily, we have for all $\overline{\eta} \in \mathcal{S}_{k}(A)$, 
\begin{equation}
  \inf_{\eta_n \in \mathcal{S}_{k,C_v}(A_n)} d_{LP}(\overline{\eta}, \eta_n) \leq d_{LP}(\overline{\eta}, \overline{\eta}_n) \leq \overline{\eps}.
\end{equation}
Thus we also have:
\begin{equation}
  \sup_{\eta \in \mathcal{S}_{k,C_v}(A)} \inf_{\eta_n \in \mathcal{S}_{k,C_v}(A_n)} d_{LP}(\eta, \eta_n) \leq \overline{\eps}.
\end{equation}

\paragraph{Proof via Assumption \ref{assumption:lipschitz_to_lipschitz_whp}.} Here, we use the other triangle inequality to get for each $x$:
\begin{align}
  |Af_j(x) - A_n f'_j(y_x)| &\leq n \int_{y_x - \frac{1}{n}}^{y_x} \left|Af_j(x) - Af_j(z)\right| + \left|A f_j(z) - A\tilde{f}_j(z)\right| \D \lambda(z)\\
                            &\leq  n \int_{y_x - \frac{1}{n}}^{y_x} \left|Af_j(x) - Af_j(z)\right| \D \lambda(z) + n \int_{y_x - \frac{1}{n}}^{y_x}\left|A f_j(z) - A\tilde{f}_j(z)\right| \D \lambda(z).
\end{align}

Define the following events:
\begin{align}
  \mathcal{E}_{j,A}^1(\eps) &:= \left\{x : n \int_{y_x - \frac{1}{n}}^{y_x}\left|A f_j(z) - A\tilde{f}_j(z)\right| \D \lambda(z) > \frac{\eps}{2\sqrt{2k}}\right\}\\
  \mathcal{E}_{j,A}^2(\eps) &:= \left\{x \not \in E : n \int_{y_x - \frac{1}{n}}^{y_x} \left|Af_j(x) - Af_j(z)\right| \D \lambda(z) > \frac{\eps}{2\sqrt{2k}}\right\} 
\end{align}

Then it is clear that $\mathcal{E}_{j,A} \subseteq \mathcal{E}_{j,A}^1 \cup \mathcal{E}_{j,A}^2 \cup E$ and thus $\lambda(\mathcal{E}_{j,A}) \leq \lambda(\mathcal{E}_{j,A}^1) + \lambda(\mathcal{E}_{j,A}^2) + \lambda(E)$.

\paragraph{Bounding $\lambda(\mathcal{E}^1_{j,A}(\eps))$ via Assumption \ref{assumption:lipschitz_to_lipschitz_whp}.}
Because of Assumption \ref{assumption:lipschitz_map}, we have:
$\|A\tilde{f}_j - Af_j\|_2 \leq C_A \|\tilde{f}_j - f_j\|_2$. By $L^{p}$ norms inequality, we have 
\begin{equation}
  \|A\tilde{f}_j - Af_j\|_1 \leq C_A \|\tilde{f}_j - f_j\|_2 \leq \frac{C_A C_v}{n}.
\end{equation}
where the last inequality is due to 
\begin{align}
  \|\tilde{f}_j - f_j\|^{2}_2 &= \int_0^{1} (\tilde{f}_j(x) - f_j(x))^{2} \D \lambda(x) \\
                              &= \sum_{u \in [n] / n} \int_{u -1 / n}^{u} (f'_j(u) - f_j(x))^{2} \D \lambda(x) \\
                              &= \sum_{u \in [n] / n} \int_{u - 1 / n}^u \left(n\int_{u - 1 / n}^u f_j(z) \D \lambda(z) - f_j(x) \right)^2  \D \lambda(x)\\
                              &\leq \sum_{u \in [n] / n} \int_{u - 1 / n}^u n^2 \left(\int_{u - 1 / n}^u |f_j(z) - f_j(x)| \D \lambda(z)\right)^2 \D \lambda(x)\\
                              &\leq \frac{C_v^2}{n^2}.
\end{align}

Consider:
\begin{align}
  \lambda(\mathcal{E}^1_{j,A}(\eps)) \cdot \frac{\eps}{2\sqrt{2k}}&\leq\int_0^1 n \int_{y_x - \frac{1}{n}}^{y_x}\left|A f_j(z) - A\tilde{f}_j(z)\right| \D \lambda(z) \D \lambda(x)\\
                                                                  &= \sum_{u \in [n] / n}  \int_{u - 1 / n}^u \left|A f_j(z) - A\tilde{f}_j(z)\right| \D \lambda(z) \\
                                                                  &= \|A f_j - A\tilde{f}_j\|_1 \leq \frac{C_AC_v}{n}.
\end{align}
Thus selecting $\eps > \sqrt{\frac{C_AC_v 4k \sqrt{2k}}{n}} = 2^{\frac{5}{4}} k^{\frac{3}{4}} (C_AC_v / n)^{\frac{1}{2}}$ gives $\lambda(\overline{\mathcal{E}}_{j,A}( \eps)) \leq \frac{\eps}{2k}$.

\paragraph{Bounding $\lambda(\mathcal{E}^2_{j,A}(\eps))$ via Assumption \ref{assumption:lipschitz_to_lipschitz_whp}.}
Notice that $f_j$ is Lipschitz in $[0,1]$ and thus by Assumption \ref{assumption:lipschitz_to_lipschitz_whp}, $Af_j$ is  $C_v$-Lipschitz in $[0,1] \backslash E$. Therefore we have, for $x \in [0,1] \backslash E$ and $z \in (y_x - 1 / n, y_x] \backslash E$ and:
\begin{align}
  |Af_j(x) - Af_j(z)|  &\leq C_v |x - z| \leq \frac{C_v}{n}.
\end{align}

When $z \in E$, we use the second condition in Assumption \ref{assumption:lipschitz_to_lipschitz_whp} to get $\|Af_j \mathbbm{1}_E\|_1 \leq \frac{1}{n}$ and thus for any $x \not \in E$:
\begin{align}
  &n \int_{y_x - 1 / n}^{y_x} \left|A f_j(x) - A f_j(z) \right| \D \lambda(z) \\
  =& n \int_{(y_x - 1 / n, y_x] \cap E} \left|A f_j(x) - A f_j(z) \right| \D \lambda(z) +  n \int_{(y_x - 1 / n, y_x] \backslash E} \left|A f_j(x) - A f_j(z) \right| \D \lambda(z)\\
\leq&   \frac{C_v}{n} + n \int_{(y_x - 1 / n, y_x] \cap E} \left|A f_j(x)\right|+\left|A f_j(z) \right| \D \lambda(z)\\
\leq& \frac{C_v}{n} + |Af_j(x)| n \lambda(E_x) + \|Af_j \mathbbm{1}_{E_x}\|_1\\
\leq& \frac{C_v + 1}{n} + |Af_j(x)| n \lambda(E_x),
\end{align}
where $E_x = (y_x - 1 / n, y_x] \cap E$.

For $x \in \mathcal{E}_{j,A}^2(\eps)$, we have: 
\begin{equation}
  \frac{C_v + 1}{n} + |Af_j(x)| n \lambda(E_x) > \frac{\eps}{2\sqrt{2k}},
\end{equation}
or equivalently,
\begin{equation}
  |Af_j(x)| > \frac{1}{n\lambda(E_x)}\left(\frac{\eps}{2\sqrt{2k}} - \frac{C_v + 1}{n}\right)
\end{equation}

Since $1 / n \geq \|Af_j \mathbbm{1}_E\|_1$, we have:
\begin{align}
  1 &\geq \left(\frac{\eps}{2\sqrt{2k}} - \frac{C_v + 1}{n}\right) \int_{\mathcal{E}^2_{j,A}}\frac{1}{\lambda(E_x)} \D \lambda(x) + 0\\
    &\geq \left(\frac{\eps}{2\sqrt{2k}} - \frac{C_v + 1}{n}\right) \sum_{u \in  [n] / n}\frac{1}{\lambda(E_u)} \int_{\mathcal{E}_{j,A}^2 \cap (u - 1 / n, u]} 1 \D \lambda(x)\\
    &\geq \left(\frac{\eps}{2\sqrt{2k}} - \frac{C_v + 1}{n}\right) n \lambda(\mathcal{E}_{j,A}^2).
\end{align}

Thus choosing $\eps > 4\sqrt{2k} (C_v + 1) / n$ means that:
\begin{align}
  1 \geq \frac{\eps}{4\sqrt{2k}} n \lambda(\mathcal{E}_{j,A}^2),
\end{align}
or
\begin{align}
  \lambda(\mathcal{E}_{j,A}^2) \leq \frac{4\sqrt{2k}}{\eps n}. 
\end{align}

Finally, choosing $\eps > \sqrt{\frac{16k\sqrt{2k}}{n}} = \frac{2^{\frac{9}{4}} k^{\frac{3}{4}}}{n^{\frac{1}{2}}}$ makes $\lambda(\mathcal{E}_{j,A}^2) \leq \frac{\eps}{4k}$; and choosing $\eps > \frac{4k}{n}$ makes  $\lambda(E) < \frac{1}{n} < \frac{\eps}{4k}$.

\paragraph{Putting everything together via Assumption \ref{assumption:lipschitz_to_lipschitz_whp}.}
Thus, we can choose $\overline{\eps} = 8k\left(\sqrt{\frac{C_AC_v + 1}{n}} + \frac{C_v + 1}{n}\right) $ to get:
\begin{equation}
  \lambda(\mathcal{E}'(\eps)) \leq \sum_{j = 1}^k \frac{\overline{\eps}}{k} = \overline{\eps},
\end{equation}
which allows us to conclude:
\begin{equation}
  d_{LP}(\overline{\eta}, \overline{\eta}_n) \leq \overline{\eps}.
\end{equation}

Since $\overline{\eta}$ was chosen arbitrarily, we have for all $\overline{\eta} \in \mathcal{S}_{k}(A)$, 
\begin{equation}
  \inf_{\eta_n \in \mathcal{S}_{k,C_v}(A_n)} d_{LP}(\overline{\eta}, \eta_n) \leq d_{LP}(\overline{\eta}, \overline{\eta}_n) \leq \overline{\eps}.
\end{equation}
Thus we also have:
\begin{equation}
  \sup_{\eta \in \mathcal{S}_{k,C_v}(A)} \inf_{\eta_n \in \mathcal{S}_{k,C_v}(A_n)} d_{LP}(\eta, \eta_n) \leq \overline{\eps}.
\end{equation}

\paragraph{Bounding $\sup_{\eta_n \in \mathcal{S}_{k,C_v}(A_n)} \inf_{\eta \in \mathcal{S}_{k,C_v}(A)} d_{LP}(\eta, \eta_n)$.} In this direction, we proceed identically, but now choose $\overline{\eta}_n$ arbitrarily in $\mathcal{S}_{k, C_v}(A_n)$. By definition of  $(k,C_v)$-profiles, there exists a tuple  $F = (f_1, \ldots, f_k)$ each in $\mathcal{F}_{n, \text{reg}(C_v)}$ such that $\overline{\eta}_n = \mathcal{D}_{A_n}(F)$. 

Construct 
\begin{align}
  F' &:= \left\{\mathcal{F}_{\text{reg}(C_v)} \ni f'_j: x \mapsto  (1 - y_xn + xn) f_j(y_x) + (y_xn -xn) f_j(y_x - 1 / n)\right\},\\
  \overline{\eta} &= \mathcal{D}_A(F'),
\end{align} where $y_x =\left\lceil xn \right\rceil / n$ for each $j \in [k]$. That $f'_j \in \mathcal{F}_{\text{reg}(C_v)}$ is asserted in Lemma \ref{lemma:extension_of_bounded_range_is_bounded}. Intuitively, $f'_j$ is the continuous piecewise linear function that interpolates  $f_j$, created by joining  $f_j(u - 1 / n)$ and $f_j(u)$ with a line segment for each  $u \in  [n] / n$. Note also that $f'_j(u) = f_j(u)$ for all  $u \in [n] / n$.

\paragraph{Bounding $d_{LP}(\overline{\eta}, \overline{\eta}_n)$.}
As with the previous direction, we start with an arbitrary $U \in \mathcal{B}_k$ and write down the differences:
\begin{align}
  \overline{\eta}_n(U) - \overline{\eta}(U^{\eps}) &= \sum_{u \in [n] / n} \int_{u - 1 / n}^u \mathbbm{1}_{F_{A_n}(u) \in U}  -\mathbbm{1}_{F'_A(x) \in U^{\eps}}\D \lambda(x),\label{eqn:xu_yueps2}\\
 \overline{\eta}(U) - \overline{\eta}_n(U^{\eps}) &= \sum_{u \in [n] / n} \int_{u - 1 / n}^{u} \mathbbm{1}_{F'_A(x) \in U} -\mathbbm{1}_{F_{A_n}(u) \in U^{\eps}}  \D \lambda(x)\label{eqn:yu_xueps2}.
\end{align}

Define the events:
\begin{align}
  \mathcal{E}^1_U(\eps) &= \{x  : F_{A_n}(y_x) \in U \wedge F'_{A}(x) \not \in U^{\eps} \}\\
  \mathcal{E}^2_{U}(\eps) &= \{x  : F'_{A}(x) \in U \wedge F_{A_n}(y_x) \not \in U^{\eps} \}\\
  \mathcal{E}'(\eps) &= \{x : \|F_{A_n}(y_x)  - F'_{A}(x)\|_2 > \eps\}\\
  \mathcal{E}_j(\eps) &= \{x : |f_j(y_x) - f'_j(x)| > \eps / \sqrt{2k} \}, &j = 1..k\\
\mathcal{E}_{j,A}(\eps) &= \{x : |A_nf_j(y_x) - Af'_j(x)| > \eps / \sqrt{2k} \}, &j = 1..k.
\end{align}

Using this notation, one also has
\begin{align}
  \overline{\eta}(U) - \overline{\eta}_n(U^{\eps}) \leq \lambda(\mathcal{E}^1_U) \qquad \text{and} \qquad \overline{\eta}_n(U) - \overline{\eta}(U^{\eps}) \leq \lambda(\mathcal{E}^2_U).
\end{align}

It is straightforward to see $\mathcal{E}^l_U(\eps) \subseteq \mathcal{E}'(\eps), l = 1,2$ for any $U$ by definition of  $U^\eps$. Furthermore,  
\begin{equation}
  \mathcal{E}'(\eps) \subseteq \bigcup_{j = 1}^k \mathcal{E}_j(\eps) \cup \mathcal{E}_{j,A}(\eps),
\end{equation}
since if all $2k$ dimensions are bounded in absolute value by $\eps / \sqrt{2k}$ then the Euclidean distance of the vector is bounded by $\eps$. 

Therefore it suffices to bound $\lambda(\mathcal{E}_j(\eps)) + \lambda(\mathcal{E}_{j,A}(\eps))$ for each $j \in [k]$.

\paragraph{Bounding $\lambda(\mathcal{E}_j(\eps))$.}
Since $f'_j$ is  $C_v$-Lipschitz for all $j$ (Lemma \ref{lemma:extension_of_bounded_range_is_bounded}), we have:
\begin{align}
  |f_j(y_x) - f'_j(x)| &= |f'_j(y_x) - f'_j(x)| \leq \frac{C_v}{n}.
\end{align} 

Thus, choosing $\eps > \sqrt{2k} C_v / n$ means that  $\lambda(\mathcal{E}_{j}(\eps)) = 0$. (We can tighten this bound by only assuming that $f_j$ is  $C_v$-Lipschitz outside a set of small measure.)

\paragraph{Bounding $\lambda(\mathcal{E}_{j,A})$.}
Let $\mathcal{F}_{[-1,1]} \ni \tilde{f}$ be the extension of $f$ defined as  $\tilde{f}(x) = f(\left\lceil xn \right\rceil / n)$ for all $x \in [0,1]$. Note that $\tilde{f}$ is not continuous in general and hence not Lipschitz. We have for any $x \in [0,1]$:
\begin{align}
  |A f'_j(x) - A_nf_j(y_x)| &= \left| Af'_j(x) - n\int_{y_x - \frac{1}{n}}^{y_x} A\tilde{f}_j(z) \D \lambda(z) \right|\\
                            &\leq n \int_{y_x - \frac{1}{n}}^{y_x} \left|Af'_j(x) - A\tilde{f}_j(z)\right| \D \lambda(z),
\end{align}
where we used uniformity of $\lambda$ and triangle inequality. The last thing that we need to show is:
\begin{align}
  \|f'_j - \tilde{f}_j\|_2^2 &= \int_0^1 ((1 - y_xn + x n) f(y_x) + (xn - y_x n) f(y_x - 1 / n) -  f(y_x))^2\D \lambda(x)\\
                             &= \int_0^1 n^2 (y_x - x)^2 (f(y_x) - f(y_x - 1 / n))^2 \D \lambda(x) \\
                             &\leq \frac{C_v^2}{n^2}.
\end{align}
From here, by a word-for-word argument, we can show that the same choice of $\overline{\eps}$ does the trick to make $\lambda(\mathcal{E}'(\eps)) \leq \overline{\eps}$. This works since we only use the fact that $f'_j \in \mathcal{F}_{\text{reg}(C_v)}$ as well as assumption conditions in the previous proof.

\paragraph{Bounding $d_M$.} We have:
\begin{equation}
  d_{H}(\mathcal{S}_k(A), \mathcal{S}_k(A_n)) = \max(\sup_{\overline{\eta}_n \in \mathcal{S}_k(A_n)} \inf_{\eta \in \mathcal{S}_k(A)} d_{LP}(\overline{\eta}_n, \eta), \sup_{\overline{\eta} \in \mathcal{S}_k(A)} \inf_{\eta_n \in \mathcal{S}_k(A_n)} d_{LP}(\overline{\eta}, \eta_n) )\leq \overline{\eps}.
\end{equation}

Therefore,
\begin{align}
  d_M(A, A_n) \leq \sum_{k = 1}^\infty \frac{8k\left(\sqrt{\frac{C_AC_v + 1}{n}} + \frac{C_v + 1}{n}\right)}{2^k} \leq 8\left(\sqrt{\frac{C_A C_v + 1}{n}} + \frac{2C_c + C_v + 1}{n}\right).
\end{align}
\end{proof}

\begin{lemma}\label{lemma:restriction_of_bounded_range_is_bounded}
  Fix $n \in \mathbb{N}$. Fix $f \in \mathcal{F}_{\text{reg}(C_v)}$. Define the restriction $f': \frac{1}{n}[n] \to [-1,1]: u \mapsto n\int_{u - 1 / n}^u f(z) \D \lambda(z)$. Then $f' \in \mathcal{F}_{n,\text{reg}(C_v)}$.
\end{lemma}
\begin{proof}
  Firstly, we have for any $u \in [n]  / n$:
  \begin{align}
    |f'(u)| \leq n \int_{u - 1 / n}^u |f(z)| \D \lambda(z) \leq n \int_{u - 1 / n}^u  1 \D \lambda(z) = 1.
  \end{align}

  Therefore $\|f'\|_{L^2([n] / n)} \leq n / n = 1$ and thus  $f'$ is measurable.  Finally, for any $u < u' \in [n] / n$:
  \begin{align}
    |f'(u) - f'(u')|  \leq n\int_{u - 1 / n}^u \left|f(z) - f(z + (u' - u))\right| \D \lambda(z) \leq C_v(u' - u),
  \end{align}
  where we use Lipschitz property of $f$ in the last inequality. This shows that  $f'$ is also $C_v$-Lipschitz.
\end{proof}
\begin{remark}
  The restriction to $\mathcal{F}_{[-1,1]}$ is necessary for this lemma to work because $L^2([0,1])$ functions can blow up near $0$.
\end{remark}

\begin{lemma}\label{lemma:extension_of_bounded_range_is_bounded}
  Fix $n \in \mathbb{N}$. Fix $f \in \mathcal{F}_{n,\text{reg}(C_v)}$. Let $y_x := \left\lceil xn \right\rceil / n$ for any $x \in [0,1]$. Define the extension $f': x \mapsto (1 - y_x n + xn) f(y_x) + (xn - y_x n) f(y_x - 1 / n)$. Then $f' \in \mathcal{F}_{\text{reg}(C_v)}$.
\end{lemma}
\begin{proof}
  Firstly, since $f'$ linearly interpolates between points of  $f$, its range cannot exceed that of  $f$. The restricted range immediately implies that the $L^2$ norm is bounded by $1$ since the support is also in  $[0,1]$. Finally, for any  $x < x' \in [0,1]$, if there is a $u$ such that $x,x' \in  [u - 1 / n, u)$ then the fact that the interpolation is linear means that the line segment from $f'(x)$ to $f'(x')$ shares the same slope as that from $f(u)$ to $f(u - 1 /n)$, which is at most $C_v$ since $f \in \mathcal{F}_{n, \text{reg}(C_v)}$. Otherwise, there exists a $u < u' \in [n] / n$ such that $x \in [u - 1 / n, u)$ and $x' \in [u' - 1 / n, u')$. We have:
  \begin{equation}
    |f'(x) - f'(x')| \leq |f'(x) - f'(u)| + |f'(u) - f'(u' - 1 / n)| + |f'(u'- 1 / n) - f'(x')| \leq C_v (x' - x),
  \end{equation} which proves Lipschitzness of $f'$.
\end{proof}

\section{Omitted proofs for Section \ref{section:graphopnn_theorems}}\label{appendix:proofs_from_graphopnn}

The following result characterizes the behaviors of Assumptions \ref{assumption:lipschitz_map}, \ref{assumption:constant_to_lipschitz}, \ref{assumption:constant_to_lipschitz_whp} and \ref{assumption:lipschitz_to_lipschitz_whp} under addition, multiplication by scalars, power, and element-wise composition with a $1$-Lipschitz map:
\begin{lemma}\label{lemma:regularity_of_assumptions}
  Fix  $k \in \mathbb{N}$ and $\alpha \in \R$. Recall that $\rho: \R\to\R$ is a $1$-Lipschitz map. Let $A_1, A_2: \mathcal{F} \to \mathcal{F}$ satisfy Assumption \ref{assumption:lipschitz_map} with constant $C_A^1$ and  $C_A^2$ respectively; and Assumption  \ref{assumption:constant_to_lipschitz} or \ref{assumption:constant_to_lipschitz_whp} or \ref{assumption:lipschitz_to_lipschitz_whp} with constant $C_c^1$ and  $C_c^2$ respectively with common resolution set $ \mathcal{N}$. Then: 
\begin{enumerate}
  \item $A_1 + A_2$ satisfy Assumptions \ref{assumption:lipschitz_map} and \ref{assumption:constant_to_lipschitz} or \ref{assumption:constant_to_lipschitz_whp} or \ref{assumption:lipschitz_to_lipschitz_whp} with constant  $(C_A^1 + C_A^2)$ and  $(C_c^1 + C_c^2)$ respectively.
  \item  $\alpha A_1$ satisfies Assumption \ref{assumption:lipschitz_map} and \ref{assumption:constant_to_lipschitz} or \ref{assumption:constant_to_lipschitz_whp} or \ref{assumption:lipschitz_to_lipschitz_whp} with constant  $|\alpha| C_A^1$ and  $|\alpha| C_c^1$ respectively.
  \item  $\rho A_1$ (where the composition is done element-wise) satisfies Assumption \ref{assumption:lipschitz_map} and \ref{assumption:constant_to_lipschitz} or \ref{assumption:constant_to_lipschitz_whp} or \ref{assumption:lipschitz_to_lipschitz_whp} with constant  $C_A^1$ and  $C_c^1$ respectively.
  \item Furthermore, if $C_c^1 = 0$ then $A_2 \circ A_1$ satisfies Assumption \ref{assumption:lipschitz_map} and \ref{assumption:constant_to_lipschitz} \ref{assumption:constant_to_lipschitz_whp} or \ref{assumption:lipschitz_to_lipschitz_whp} with constant  $C_A^1 C_A^2$ and $C_c^2$ respectively.
\end{enumerate}
\end{lemma}

\begin{proof}[Proof of Lemma \ref{lemma:regularity_of_assumptions}]
  To recall,  $A_1, A_2: \mathcal{F} \to \mathcal{F}$ are $P$-operators that satisfies  Assumption \ref{assumption:lipschitz_map} with constant  $C_A^1$ and  $C_A^2$ respectively and Assumption \ref{assumption:constant_to_lipschitz} or \ref{assumption:constant_to_lipschitz_whp} or \ref{assumption:lipschitz_to_lipschitz_whp} with constant  $C_c^1$ and  $C_c^2$ respectively with common resolution set  $\mathcal{N}$. We now show each part of the Lemma.
  \begin{enumerate}
    \item
      We have, for any $f, g \in \mathcal{F}$:
      \begin{align}
        \|(A_1 + A_2) f - (A_1 + A_2)g\|_2 &\leq \|A_1 f - A_1 g\|_2 + \|A_2 f - A_2 g\|_2\\
                                           &\leq C_A^1 \|f - g\|_2 + C_A^2 \|f - g\|_2,
      \end{align}
      where the first line is triangle inequality and the second line is Assumption \ref{assumption:lipschitz_map}.

    Let $f \in \mathcal{F}_{\text{reg}(C_v)}$ piecewise constant on intervals $[n] / n$ and $x,y \in (u - 1 / n, u]$ for some $n \in \mathcal{N}$ and $u \in [n]$, we have:
      \begin{align}
        |(A_1 + A_2)f(x) - (A_1 + A_2)f(y)| &\leq |[A_1f](x) - [A_1f](y)| +  |[A_2f](x) - [A_2f](y)|\\
                                            &\leq C_c^1 |x - y| + C_c^2 |x - y|,
      \end{align}
      where the first line is triangle inequality and the second line is Assumption \ref{assumption:constant_to_lipschitz}.
    \item 
      We have, for any $f, g \in \mathcal{F}$:
      \begin{align}
        \|(\alpha A_1) f - (\alpha A_1)g\|_2 &\leq |\alpha| \cdot \|A_1 f - A_1 g\|_2\\
                                           &\leq |\alpha| C_A^1 \|f - g\|_2, 
      \end{align}
      where the first line is property of norm and the second line is Assumption \ref{assumption:lipschitz_map}.

      Let $f \in \mathcal{F}_{\text{reg}(C_v)}$ piecewise constant on intervals $[n] / n$, and $x,y \in (u - 1 / n, u]$ for some $n \in \mathcal{N}$ and $u \in [n]$, we have:
      \begin{align}
        |(\alpha A_1)f(x) - (\alpha A_1)f(y)| &\leq |\alpha [A_1f](x) - \alpha [A_1f](y)| \\
                                            &\leq|\alpha| C_c^1 |x - y|,
      \end{align}
      where the second line is Assumption \ref{assumption:constant_to_lipschitz}.
    \item We have, for any $f, g \in \mathcal{F}$:
      \begin{align}
        \|(\rho A_1) f - (\rho A_1)g\|_2 &\leq \|A_1f - A_1g\|_2\\
                                           &\leq C_A^1 \|f - g\|_2, 
      \end{align}
      where the first line is Lipschitz property of $\rho$ and the second line is Assumption \ref{assumption:lipschitz_map}.

    Let $f \in \mathcal{F}_{\text{reg}(C_v)}$ piecewise constant on intervals $[n] / n$ and $x,y \in (u - 1 / n, u]$ for some $n \in \mathcal{N}$ and $u \in [n]$, we have:
      \begin{align}
        |(\rho A_1)f(x) - (\rho A_1)f(y)| &\leq |\rho([A_1f](x)) - \rho([A_1f](y))| \\
                                            &\leq |A_1 f(x) - A_1f(y)| \leq C_c^1 |x - y|,
      \end{align}
      where the second line is Lipschitz property of $\rho$ and Assumption \ref{assumption:constant_to_lipschitz}.
    \item We have, for any $f, g \in \mathcal{F}$:
      \begin{align}
        \|(A_2 A_1) f - (A_2 A_1)g\|_2 &\leq C_A^2\|A_1f - A_1g\|_2\\
                                           &\leq C_A^1 C_A^2\|f - g\|_2, 
      \end{align}
      where the first line is Lipschitz property of $A_2$ and the second line is that of $A_1$.

      Let $f \in \mathcal{F}_{\text{reg}(C_v)}$ piecewise constant on intervals $[n] / n$. Since $A_1$ send constant pieces to constant pieces, the final implication follows direction from Assumption \ref{assumption:constant_to_lipschitz} of  $A_2$. 
  \end{enumerate}
\end{proof}

\begin{lemma}\label{lemma:recurrent1}
  Fix $n \in \mathcal{N}, k \in [K] \cup \{0\}, A: \mathcal{F} \to \mathcal{F}$ satisfies Assumption \ref{assumption:lipschitz_map} with constant $C_A$ and Assumption \ref{assumption:constant_to_lipschitz} with constant  $C_c$ and resolution set  $\mathcal{N}$. Let $n \in \mathcal{N}, f_1 \in \mathcal{F}_n, f_2 \in \mathcal{F}$. Recall that $\tilde{f} \in \mathcal{F}$ denotes the extension of  $f \in \mathcal{F}_n$ to  $[0,1]$ as $\tilde{f}(x) = f (\left\lceil xn \right\rceil) / n$. If $\|\widetilde{f_1} - f_2\|_2 \leq M$ for some positive constant $M$ then 
  \begin{equation}
    \|\widetilde{A_n^k f_1} - A^k f_2\|^2_2 \leq \frac{3^{k + 1} kC_c^2 C_A^{2k}}{n^2} + 3^k C_A^{2k}M.
  \end{equation} 

  If $A$ satisfies Assumption \ref{assumption:constant_to_lipschitz_whp} with constant $C_c$ instead then the  bound becomes:
  \begin{equation}
      \|\widetilde{A_n^k f_1} - A^k f_2\|^2_2 \leq \frac{3^{k + 1} k(C_c+1)^2 C_A^{2k}}{n^2} + 3^k C_A^{2k}M.
  \end{equation}

  If $A$ satisfies Assumption \ref{assumption:lipschitz_to_lipschitz_whp} instead then the bound becomes:
  \begin{equation}
      \|\widetilde{A_n^k f_1} - A^k f_2\|^2_2 \leq \frac{3^{k + 1} k(C_v+1)^2 C_A^{2k}}{n^2} + 3^k C_A^{2k}M.
  \end{equation}
  \end{lemma}
\begin{proof}
  When $k = 0$, the bound is vacuously true. We will hide the measure in integrals when it is clear from context in this proof.

  Assume that the bound is correct up to some  $k - 1$. 
  We have:
  \begin{align}
    &\|\widetilde{A_n^k f_1} - A^k f_2\|^2_2 \\
    = &\sum_{u \in 1 / n [n]} \int_{u - \frac{1}{n}}^u (A^k_n f_1(u) - A^k f_2(z) )^2\D z \\
    = &\sum_{u \in 1 / n [n]} \int_{u - \frac{1}{n}}^u \left(n\int_{u - \frac{1}{n}}^u A \widetilde{[A^{k - 1}_n f_1]}(y) \D y - A^k f_2(z) \right)^2\D z
  \end{align}

  Here we can use Jensen's inequality to get:
  \begin{align}
    &\|\widetilde{A_n^k f_1} - A^k f_2\|^2_2 \\
    = &n \sum_{u \in 1 / n [n]} \int_{u - \frac{1}{n}}^u \int_{u - \frac{1}{n}}^u \left(A \widetilde{[A^{k - 1}_n f_1]}(y) - A[A^{k - 1}f_2](z)  \right)^2\D y \D z\\
  \end{align}
  
We proceed slightly differently based on the exact assumptions that we have. If $A$ satisfies Assumption \ref{assumption:constant_to_lipschitz} or \ref{assumption:constant_to_lipschitz_whp} then:
  \begin{align}
&\|\widetilde{A_n^k f_1} - A^k f_2\|^2_2 \\
    \leq &n \sum_{u \in 1 / n [n]} \int_{u - \frac{1}{n}}^u \nonumber \\
    \qquad&\int_{u - \frac{1}{n}}^u \left(|A \widetilde{[A^{k - 1}_n f_1]}(y) - A \widetilde{[A^{k - 1}_n f_1]}(z)| + |A \widetilde{[A^{k - 1}_n f_1]}(z) - A[A^{k - 1} f_2](z)|   \right)^2\D y \D z\\
    \leq  &n \sum_{u \in 1 / n [n]} \int_{u - \frac{1}{n}}^u \nonumber\\
    \qquad&\int_{u - \frac{1}{n}}^u \left(\frac{C_c + 1}{n} + |A\widetilde{[A_n^{k - 1} f_1]}(z)| n \lambda(E_z) + |A \widetilde{[A^{k - 1}_n f_1]}(z) - A[A^{k - 1}f_2](z)| \right)^2\D y  \D z\\
    \leq  &\int_0^1\left( \frac{C_c + 1}{n} +  |A\widetilde{[A_n^{k - 1} f_1]}(z)| n \lambda(E_z)+ |A \widetilde{[A^{k - 1}_n f_1]}(z) - A[A^{k - 1} f_2](z)|  \right)^2\D z\\
    \leq &\frac{3(C_c+ 1)^2}{n^2} +  3 \int_0^1 (A\widetilde{[A_n^{k - 1} f_1]}(z) n \lambda(E_z))^2 \D z + 3 \|A \widetilde{[A^{k - 1}_n f_1]} - A[A^{k - 1} f_2]\|_2^2 \D z.
    \end{align}

  Here, we give a heuristic argument while the formal argument is exactly similar to that in Theorem \ref{theorem:approximation_via_discretization}. The term $ 3 \int_0^1 (A\widetilde{[A_n^{k - 1} f_1]}(z) n \lambda(E_z))^2 \D z$ can be made as close to $0$ as possible simply by changing the requirement on $\lambda(E)$ to be smaller and smaller (see the discussion after Assumption \ref{assumption:constant_to_constant_whp}). At the same time, our assumptions ensure that since $\widetilde{A_n^{k - 1}f_1}$ is a piecewise constant function, the action of $A$ on it cannot become too wild (Lipschitz outside of $E$ and bounded $L^1$ norm inside $E$). This heuristic argument, when made rigorous can help us conclude that:
   \begin{align}
 \|\widetilde{A_n^k f_1} - A^k f_2\|^2_2 \leq \frac{3(C_c+ 1)^2}{n^2} + 3C_A^2 \|\widetilde{A^{k - 1}_n f_1} - A^{k - 1} f_2\|_2^2 
  \end{align}
  
  Plug in the inductive hypothesis and solve the recurrent to get the bound.

  Similarly, if we are instead using Assumption \ref{assumption:lipschitz_to_lipschitz_whp} then we use the other triangle inequality to conclude:
  \begin{align}
&\|\widetilde{A_n^k f_1} - A^k f_2\|^2_2 \\
    \leq &n \sum_{u \in 1 / n [n]} \int_{u - \frac{1}{n}}^u \nonumber \\
    \qquad&\int_{u - \frac{1}{n}}^u \left(|A \widetilde{[A^{k - 1}_n f_1]}(y) - A [A^{k - 1} f_1](y)| + | A [A^{k - 1} f_1](y) - A[A^{k - 1} f_2](z)|   \right)^2\D y \D z\\
    \leq  &2n \sum_{u \in 1 / n [n]} \int_{u - \frac{1}{n}}^u \int_{u - \frac{1}{n}}^u \left(A \widetilde{[A^{k - 1}_n f_1]}(y) - A [A^{k - 1} f_1](y)\right)^2 \D y \D z \nonumber\\
    \qquad&+2n \sum_{u \in 1 / n [n]} \int_{u - \frac{1}{n}}^u \int_{u - \frac{1}{n}}^u\left(\frac{C_v + 1}{n} + |A[A^{k - 1} f_2](z)| n \lambda(E_z) \right)^2\D y  \D z\\
    \leq &\frac{3(C_v+ 1)^2}{n^2} +  3 \int_0^1 (A[A^{k - 1} f_1](z) n \lambda(E_z))^2 \D z + 3 \|A \widetilde{[A^{k - 1}_n f_1]} - A[A^{k - 1} f_2]\|_2^2.
    \end{align}

    Again, we will make use of a heuristic argument to argue that the middle term can be controlled by controlling $\lambda(E)$ since $A$ sends Lipschitz functions to Lipschitz functions outside of $E$ and has bounded $L^1$ norm inside of $E$. Therefore:
    \begin{align}
        \|\widetilde{A_n^k f_1} - A^k f_2\|^2_2 \leq &\frac{3(C_v+ 1)^2}{n^2}  + 3 \|A \widetilde{[A^{k - 1}_n f_1]} - A[A^{k - 1} f_2]\|_2^2.
    \end{align}

    Solve the recurrent to get the result in the statement of the Lemma.
\end{proof}

\begin{lemma}\label{lemma:recurrent2}
  In the same setting as Lemma \ref{lemma:recurrent1}, recall that a tilde over a function in $\mathcal{F}_n$ denotes its extension to $\mathcal{F}$. If $A$ satisfies Assumption \ref{assumption:lipschitz_map} with constant $C_A$ and Assumption \ref{assumption:constant_to_lipschitz} with constant $C_c$ and resolution set $\mathcal{N}$. Let $\Phi(h, A, \cdot) = \rho \left( \sum_{g \in n_l} \sum_{k = 0}^{K-1} A^k \Phi_g(h,A,\cdot) \right)$ for some $\Phi_g$ graphop neural network such that $\|\widetilde{\Phi_g}(h, A_n, f) - \Phi_g(h, A, \tilde{f})\|_2 < M$ for all $g \in [n_l]$. Then:
  \begin{equation}
  \|\widetilde{\Phi}(h, A_n, f) - \Phi(h, A, \tilde{f})\|_2^2 \leq K^2 n_l^2 3^{K} \left(Kn^{-2} C_c^2 C_A^{2K} + C_A^{2K} M \right) .
  \end{equation}

    If instead $A$ satisfies Assumption \ref{assumption:constant_to_lipschitz_whp} with constant $C_c$ then the bound becomes:
 \begin{equation}
  \|\widetilde{\Phi}(h, A_n, f) - \Phi(h, A, \tilde{f})\|_2^2 \leq K^2 n_l^2 3^{K} \left(Kn^{-2} (C_c+1)^2 C_A^{2K} + C_A^{2K} M \right) .
  \end{equation}
    If instead $A$ satisfies Assumption \ref{assumption:lipschitz_to_lipschitz_whp} then the bound becomes:
     \begin{equation}
  \|\widetilde{\Phi}(h, A_n, f) - \Phi(h, A, \tilde{f})\|_2^2 \leq K^2 n_l^2 3^{K} \left(Kn^{-2} (C_v+1)^2 C_A^{2K} + C_A^{2K} M \right) .
  \end{equation}
  
\end{lemma}
\begin{proof}
  We have:
  \begin{align}
    &\|\widetilde{\Phi}(h, A_n, f) - \Phi(h, A, \tilde{f})\|_2^2\\
    = &\sum_{u \in  \frac{1}{n}[n]} \int_{u - 1 /n }^u \left(\rho\left( \sum_g \sum_k \widetilde{A_n^k \Phi_g}(h, A_n,f) (x)  \right) -  \rho\left( \sum_g \sum_k A^k \Phi_g(h, A,\tilde{f}) (x)  \right)\right)^2 \D \lambda(x)\\
    \leq &\sum_{u \in  \frac{1}{n}[n]} \int_{u - 1 /n }^u \left(\sum_g \sum_k \widetilde{A_n^k \Phi_g}(h, A_n,f) (x) -   A^k \Phi_g(h, A,\tilde{f}) (x)\right)^2 \D \lambda(x)\\
    \leq &\sum_{u \in  \frac{1}{n}[n]} \int_{u - 1 /n }^u \left(\sum_g \sum_k \left|\widetilde{A_n^k \Phi_g}(h, A_n,f) (x) -   A^k \Phi_g(h, A,\tilde{f}) (x)\right|\right)^2 \D \lambda(x)\\
    \leq &K n_l \sum_{u \in  \frac{1}{n}[n]} \int_{u - 1 /n }^u \sum_g \sum_k \left(\left|\widetilde{A_n^k \Phi_g}(h, A_n,f) (x) -   A^k \Phi_g(h, A,\tilde{f}) (x)\right|\right)^2 \D \lambda(x)\\
    = & K n_l \sum_g \sum_k \|\widetilde{A_n^k \Phi_g}(h, A_n,f) (x) -   A^k \Phi_g(h, A,\tilde{f}) (x)\|_2^2\\
  \leq &K n_l \sum_g \sum_k \frac{3^{k + 1}kC_c^2 C_A^{2k}}{n^2} + 3^k C_A^{2k} \|\widetilde{\Phi_g}(h, A_n, f) - \Phi_g(h, A, \tilde{f})\|_2^2\\
  \leq &n^{-2}n_L^2 \cdot K^3 \cdot C_c^2 \cdot 3^{K} C_A^{2K} + K^2 \cdot n_l^2 \cdot 3^K C_A^{2K}\cdot M \\
  = &K^2 n_l^2 3^{K} \left(Kn^{-2} C_c^2 C_A^{2K} + C_A^{2K} M \right),
  \end{align}
  where in the first line, we use the fact that continuous extension commutes with finite element-wise sum and element-wise application of $\rho$, while expanding  $L^2$ norm; the second line uses Lipschitz property of  $\rho$; the third line uses triangle inequality; the fourth line uses equivalence of  $p$-norms in finite dimensional vectors; the fifth line is again  $L^2$ norm definition; the sixth line applies Lemma \ref{lemma:recurrent1} and the rest is algebra.

  To get the rest of the cases, applies different versions of Lemma \ref{lemma:recurrent1}.
\end{proof}

\begin{lemma}\label{lemma:recurrent3}
  Let $\Phi$ be an  $L$-layer graphop neural network in the same setting as Lemma \ref{lemma:recurrent1}. If $A$ satisfies Assumption \ref{assumption:constant_to_lipschitz} with constant $C_c$ then:
   \begin{equation}
   \|\widetilde{\Phi}(h, A_n, f) - \Phi(h, A, \tilde{f})\|_2 \leq n^{-1}(3^K K n_{\max}  C_A^{K})^L \max\left(C_v, 3^K K^2 C_c n_{\max} C_A^{K}\right) .
  \end{equation}
  If instead $A$ satisfies Assumption \ref{assumption:constant_to_lipschitz_whp} with constant $C_c$ then the bound becomes:
 \begin{equation}
    \|\widetilde{\Phi}(h, A_n, f) - \Phi(h, A, \tilde{f})\|_2 \leq n^{-1}(3^K K n_{\max}  C_A^{K})^L \max\left(C_v, 3^K K^2 (C_c + 1) n_{\max} C_A^{K}\right) .
  \end{equation}
    If instead $A$ satisfies Assumption \ref{assumption:lipschitz_to_lipschitz_whp} then the bound becomes:
     \begin{equation}
        \|\widetilde{\Phi}(h, A_n, f) - \Phi(h, A, \tilde{f})\|_2 \leq n^{-1}(3^K K n_{\max}  C_A^{K})^L \max\left(C_v, 3^K K^2 (C_v + 1) n_{\max} C_A^{K}\right) .
  \end{equation}
\end{lemma}
\begin{proof}
  Solve the recurrent in Lemma \ref{lemma:recurrent2}.
\end{proof}

 We now prove a more general version of Theorem \ref{theorem:graphopnn_discretization}
\begin{theorem}\label{theorem:general_graphopnn}
  Let $A \in \mathcal{F}$ satisfying Assumption \ref{assumption:lipschitz_map} with constant $C_A$ and Assumption \ref{assumption:constant_to_lipschitz} with constant  $C_c$ and resolution set  $ \mathcal{N} \subseteq \mathbb{N}$. Let $n \in \mathcal{N}$ and form the discretization $A_n$ as per Theorem \ref{theorem:approximation_via_discretization}. Let  $h$ be normalized such that  $|h| \leq 1$ element-wise and form the graphop neural network $\Phi(h,A,\cdot): \mathcal{F} \to \mathcal{F}$ and $\Phi(h,A_n,\cdot): \mathcal{F}_n \to \mathcal{F}_n$. We have the following approximation bound:
  \begin{equation}
    d_M(\Phi(h,A,\cdot), \Phi(h,A_n,\cdot)) \leq n^{-1/2} P_1\sqrt{\overline{C_A} C_v + C_c   n_{\max} C_A^{K}\cdot P_2} + C_v n^{-1}
  \end{equation}
  where $P_1 = 3^{KL}$ and $P_2 = 3^K K^2$. 

If instead $A$ satisfies Assumption \ref{assumption:constant_to_lipschitz_whp} with constant $C_c$ then the bound becomes:
 \begin{equation}
 d_M(\Phi(h,A,\cdot), \Phi(h,A_n,\cdot)) \leq n^{-1/2} P_1\sqrt{\overline{C_A} C_v + (C_c+1)   n_{\max} C_A^{K}\cdot P_2} + (C_v + 1) n^{-1}
  \end{equation}
    If instead $A$ satisfies Assumption \ref{assumption:lipschitz_to_lipschitz_whp} then the bound becomes:
     \begin{equation}
     d_M(\Phi(h,A,\cdot), \Phi(h,A_n,\cdot)) \leq n^{-1/2} P_1\sqrt{\overline{C_A} C_v} + (C_v + 1) n^{-1}
  \end{equation}

\end{theorem}
\begin{proof}
  The proof structure is similar to the proof of Theorem \ref{theorem:approximation_via_discretization}. For brevity of exposition, we will write $h$ as the (shared) set of parameters in the graphop neural networks and shorten  $\Phi := \Phi(h,A,\cdot)$ and  $\Phi_n := \Phi(h,A_n, \cdot)$. Fix $k \in \mathbb{N}$, we will bound  $\sup_{\eta \in \mathcal{S}_{k,C_v}(\Phi)} \inf_{\overline{\eta}_n \in \mathcal{S}_{k,C_v}(\Phi_n)} d_{LP}(\overline{\eta}, \overline{\eta}_n)$. 

  Fix arbitrary $\overline{\eta} \in \mathcal{S}_{k,C_v}(\Phi)$. By definition of a  $(k,C_v)$-profile, there exists  $f_1, \ldots, f_k \in L^\infty_{\text{reg}(C_v)}([0,1])$ such that $\overline{\eta} = \mathcal{D}_{\Phi}(f_1,\ldots,f_k)$. For all $j \in [k]$, let $\mathcal{F}_n \ni f'_j : [n] / n \ni u \mapsto \int_{u - 1 / n}^u f_j(z) \D \lambda(z)$. That $f'_j \in L_{\text{reg}(C_v)}^\infty([n] / n)$ is shown in Lemma \ref{lemma:restriction_of_bounded_range_is_bounded}. Set $\overline{\eta}_n = \mathcal{D}_{\Phi_n}(f_1',\ldots,f_k')$.

  \paragraph{Bounding $d_{LP}(\overline{\eta},\overline{\eta}_n)$.} Fix some $\eps > 0$ to be specified later, by definition of $d_{LP}$, we need to bound, for each $U \in \mathcal{B}_k$:
  \begin{align}
    \overline{\eta}(U) - \overline{\eta}_n(U^\eps) = \int_{0}^1 \mathbbm{1}_{(f_1(x),\ldots,\Phi f_k(x)) \in U} \D \lambda(x) - \sum_{u \in [n] / n} \frac{1}{n} \mathbbm{1}_{(f'_1(u),\ldots,\Phi_nf'_k(u)) \in U^\eps}.
  \end{align}

  Notice that since $\lambda$ is the Lebesgue measure,  one can denote  $y_x = \left\lceil xn \right\rceil / n$ and write:
  \begin{align}
    \overline{\eta}(U) - \overline{\eta}_n(U^\eps) = \int_{0}^1 \mathbbm{1}_{(f_1(x),\ldots,\Phi f_k(x)) \in U} - \mathbbm{1}_{(f'_1(y_x), \ldots, \Phi_n f'_k(y_x)) \in U^\eps} \D \lambda(x).
  \end{align}

  Since the integrand is only positive when $(f_1(x),\ldots,\Phi f_k(x)) \in U$ and $(f'_1(y_x),\ldots,\Phi_n f'_k(y_x)) \not \in U^\eps$ and this conjuction only happens when $\|(f_1(x),\ldots,\Phi f_k(x)) - (f'_1(y_x),\ldots,\Phi_n f'_k(y_x))\|_2 > \eps$, we have:
  \begin{align}
    \overline{\eta}(U) - \overline{\eta}_n(U^\eps) \leq \lambda(\mathcal{E}(\eps)) \leq \sum_{j = 1}^k \lambda(\mathcal{E}_j^0(\eps)) + \lambda(\mathcal{E}_j^1(\eps)),
  \end{align}where
  \begin{align}
    \mathcal{E}(\eps) &:= \{x : \|(f_1(x),\ldots,\Phi f_k(x)) - (f'_1(y_x),\ldots,\Phi_n f'_k(y_x))\|_2 > \eps\},\\
    \mathcal{E}_j^z(\eps) &= \{x : |\Phi^z f_j(x) - \Phi^z_n f'_j(y_x)| > \frac{\eps}{\sqrt{2k}}\}, \text{ for } z \in \{0,1\}, j \in [k],
  \end{align} where $B^0$ is the identity operator for any operator  $B$ in the appropriate space.

  \paragraph{Bounding $\lambda(\mathcal{E}^1_j(\eps))$.} Fix $j \in [k]$, we have:
\begin{align}
&\int_0^1 |\Phi f_j(x) -  \Phi_n f'_j (y_x)| \D x \leq \sqrt{\int_0^1 \left(\Phi f_j(x) - \Phi_n f'_j (y_x)\right)^2} \D x\\
= &\|\Phi f_j - \widetilde{\Phi_n}f'_j\|_{L^2}\\
\leq &n^{-1}(3^K K n_{\max}  C_A^{K})^L \max\left(C_v, 3^K K^2 C_c n_{\max} C_A^{K}\right) .
\end{align}
where the last line uses Lemma \ref{lemma:recurrent3} under Assumption \ref{assumption:constant_to_lipschitz} with constant $C_c$. Similar results are obtained for the other assumptions.

We also have:
\begin{equation}
 \int_0^1 |\Phi f_j(x) -  \Phi_n f'_j (y_x)| \D x \geq \lambda(\mathcal{E}_j^1(\eps)) \cdot \frac{\eps}{\sqrt{2k}} + 0. 
\end{equation}

Thus selecting 
\begin{equation}
    \eps > \sqrt{2k\sqrt{2k}\cdot n^{-1}(3^K K n_{\max}  C_A^{K})^L \max\left(C_v, 3^K K^2 C_c n_{\max} C_A^{K}\right)} 
\end{equation}
makes $\lambda(\mathcal{E}_j^1(\eps)) < \frac{\eps}{2k}$.
\end{proof} 

\paragraph{Bounding $\lambda(\mathcal{E}_j^0(\eps))$.} This step simply uses Lipschitzness of the graph signal to bound its discretization and is thus identical to that in Theorem \ref{theorem:approximation_via_discretization}. In short, choosing $\eps > \sqrt{2k} C_v / n$ gives $\lambda(\mathcal{E}_j^0 (\eps)) = 0$. 

\paragraph{Bounding $\sup_{\eta_n \in \mathcal{S}_{k, C_v}(\Phi_n)} \inf_{\eta \in \mathcal{S}_{k, C_v}(\Phi)}d_{LP}(\eta,\eta_n)$. } This direction is almost the same as the previous bounds, just as in Theorem \ref{theorem:approximation_via_discretization}. We will provide the complete proof in a later version of the paper but the same choice of $\eps$ works to bound this side of the Hausdorff metric.

\paragraph{Putting everything together}
Defining $\overline{C_A} = (n_{\text{max}} K C_A^K)^L$ gives a choice of:
\begin{equation}
    \overline{\eps} = n^{-1/2} P_1\sqrt{\overline{C_A} C_v + C_c   n_{\max} C_A^{K}\cdot P_2} + C_v n^{-1}
\end{equation}
where $P_1 = 3^{KL}$ and $P_2 = 3^K K^2$. 

For other choices of assumptions, we obtain similar bounds.

\section{Additional results}
\subsection{Proof of Theorem \ref{theorem:growing_profiles}}
Because of completeness of $d_H$ in the space of closed subsets of  $\mathcal{P}(\R^{2k})$ for every $k \in \mathbb{N}$, the statement in Theorem \ref{theorem:growing_profiles} is equivalent to showing that for each $k \in \mathbb{N}$, we have $\mathcal{S}_{k,L(n)}(A_n)$ converges to $\mathcal{S}_k(A)$ in $d_H$. We do this via a mollificaition  argument.
\begin{definition}[Lipschitz mollifier]
  A \emph{Lipschitz mollifier} in $\R$ is a smooth (infinitely differentiable) function  $\phi: \R\to \R$ satisfying:
  \begin{enumerate}
    \item $\int_{\R} \phi(x) \D \lambda(x) = 1$.
    \item $\lim_{\eps \to 0} \phi_\eps(x) := \lim_\eps \eps^{-1} \phi(x / \eps) = \delta(x)$ - the Dirac function. 
    \item Although not standard, we require $\phi$  to be $1$-Lipschitz and symmetric around $0$ ($\phi(x) = \phi(-x)$).
  \end{enumerate}

  Given a measureable function $f \in L^\infty_{[-1,1]}(\R)$, defines the convolution operation:
  \begin{equation}
    f \ast \phi: x \to \int_{\R} f(y) \phi(y - x) \D \lambda(y).
  \end{equation}
\end{definition}

The next result shows the existence of such a function:
\begin{lemma}
  Let $\phi: \R \to \R$ be:
  \begin{equation}
    \phi(x) = \begin{cases}
      e^{-(1 - x^2)^2} / Z&\text{ if } |x| \leq 1\\
      0 &\text{ otherwise,}
    \end{cases}
  \end{equation}
  where $Z$ is a normalization constant to make sure that  $\int_{\R} \phi \D \lambda = 1$. Then $\phi$ is a Lipschitz mollifier.
\end{lemma}
\begin{proof}
  The first property is built into the definition. The second property is obvious since the support of $\phi(x / \eps)$ is  $[-\eps,\eps]$ and thus it converges to the Dirac function as  $\eps$ goes to  $0$.  Lipschitz-ness can be seen by computing the first derivative (since the function is smooth) over $(-1,1)$ and see that it is bounded in  $[-1,1]$. Symmetry is also obvious since the function depends only on  the absolute value of its argument.
\end{proof}

Consequences of a Lipschitz mollifier include:
\begin{lemma}
  Let $\phi$ be a Lipschitz mollifier and $\eps > 0$, then for any measurable  $f \in \mathcal{F}_{[-1,1]}^\infty(\R)$,  $f \ast \phi_\eps$ is $\max(1, \eps^{-2})$-Lipschitz and $\lim_{\eps \to 0}\|f - f \ast \phi_\eps\|_2 = 0$.
\end{lemma}
\begin{proof}
  For the first part of the statement, consider:
  \begin{align}
    \left|f \ast \phi_\eps (x) - f \ast\phi_\eps(y)\right| &= \left| \int_{\R}  f(z) \phi_\eps(z - x) - f(z) \phi(z - y)\D\lambda(z) \right|\\
                                                           &\leq \int_{\R} \left|\frac{f(z)}{\eps} \left( \phi\left(\frac{z-x}{\eps}\right) - \phi\left(\frac{z-y}{\eps}\right) \right)\right|\D \lambda(z) \\
                                                           &\leq \frac{1}{\eps} \int_{\R} \left|\frac{z - x}{\eps} - \frac{z - y}{\eps}\right|\D \lambda(z)\\
                                                           &\leq \frac{|x - y|}{\eps^2}
  \end{align}

  For the second part, we have:
  \begin{align}
    |f \ast \phi_\eps(x) - f(x) | &\leq \int_{\R} \phi_{\eps}(z - x) |f(z) - f(x)| \D \lambda(z)\\
                                  &\leq \int_{\R} \phi_{\eps}(z) |f(z + x) - f(x)| \D \lambda(z),
  \end{align}
  where we use two changes of variables. 

  Square both sides and take integral over  $x$ and use Jensen inequality to get:
  \begin{align}
    \|f \ast \phi_\eps - f\|_2^2 &\leq \int_{\R}\int_{\R} (\phi_{\eps}(z))^2 (f(z + x) - f(x))^2 \D \lambda(z) \D \lambda(x).
  \end{align}

  Apply Fubini-Tonelli theorem to factorize the mollifier, we get:
  \begin{align}
    \|f \ast \phi_\eps - f\|_2^2 &\leq \int_{\R}\int_{\R}  (f(z + x) - f(x))^2 \D \lambda(x) (\phi_{\eps}(z))^2 \D \lambda(z)\\
                                 &= \int_{\R}\left(\eps^{-1}\int_{\R}  (f(z + x) - f(x))^2 \D \lambda(x)\right)  (\phi(z\eps^{-1}))^2 \D \lambda(z).
  \end{align}

  Apply a final change of variable:
  \begin{align}
    \|f \ast \phi_\eps - f\|_2^2 \leq  \int_{\R}\left(\int_{\R}  (f(z\eps + x) - f(x))^2 \D \lambda(x)\right)  (\phi(z))^2 \D \lambda(z).
  \end{align}

  Therefore, as $\eps$ goes to  $0$, the inner integrand goes to  $0$. Since  $f$ and  $f\ast \phi_\eps$ are all bounded (as  $f \in [-1,1]$) for small enough $\eps$, we can apply dominated convergence to conclude that the integral itself goes to  $0$ and thus the  $2$-norm on the left hand side also goes to  $0$.
\end{proof}

We are now ready to proceed with the proof of Theorem \ref{theorem:growing_profiles}. Given an element $\eta \in  \mathcal{S}_k(A)$, there exists a set of functions $F = \{f_1,\ldots,f_k\} \subseteq L^\infty_{[-1,1]}([0,1])$. Each of these functions can be extended to $\R$ by setting $f(x) = f(0)$ if  $x \leq 0$ and  $f(x) = f(1)$ if $x > 1$. Call the extended function  $f'$. Now we can apply mollification convolution to each of them to get a family of functions  $\{f_{j,\eps} := f'_j \ast \phi_\eps\} _{j \in [k], \eps > 0}$. Recall that we have shown $f_{j,n} := f_{j, 1 / \sqrt{L(n)}}$ to be $L(n)$-Lipschitz for each  $n \in \mathbb{N}$.  Let $f''_{j,n}$ be the restriction of  $f_{j,n}$ to $\mathcal{F}_n$. Then $f''_{j,n}$ is still $L(n)$-Lipschitz and thus we can find a profile for  $F''_n := \{f''_{1,n}, \ldots, f''_{k,n}\}$ in $ \mathcal{S}_{k,L(n)}(A_n )$.   

Furthermore, by property of mollifier and the fact that  $L(n) \to \infty$ as  $n$ goes to infinity, we have $\|f_{j, n} - f'_j\|_2 \to 0$ with $n \to \infty$.  Since $f'_j$ is constant outside $[0,1]$, we also have $\|f_{j,n,|[0,1]} - f_j\|_{\mathcal{F}} \to 0$ with $n$. Using the same proof technique as Theorem \ref{theorem:approximation_via_discretization}, we can conclude that:
\begin{equation}
  d_{LP}(\mathcal{D}_{A_n}(F''_n), \mathcal{D}_A(F)) \xrightarrow{n \to \infty} 0,
\end{equation}
and thus $\mathcal{S}_k(A) \subseteq \lim_n \overline{\mathcal{S}_{k,L(n)}(A_n)}$. For the other direction, recall that $\overline{\mathcal{S}_k(A)} = \lim_n \overline{\mathcal{S}_{k}(A_n)}$ and that $\mathcal{S}_{k,L(n)}(A_n) \subseteq \mathcal{S}_{k}(A)$. Together with completeness of $d_H$, we conclude that  $$\lim_n d_H(\mathcal{S}_k(A), \mathcal{S}_{k,L(n)}(A_n)) = 0.$$

\begin{conjecture}[Action convergence of graphop neural networks]\label{theorem:convergence_graphopnn}
  Let $(A_n)_{n \in \mathbb{N}}$ be an action convergent sequence of graphops. Then  $(\Phi(h,A_n,\cdot))_{n \in \mathbb{N}}$ is an action convergent sequence of $P$-operators.
\end{conjecture}

\end{document}